\documentclass{article}

\usepackage[preprint]{main}

\usepackage[utf8]{inputenc}
\usepackage[T1]{fontenc}
\usepackage{url}
\usepackage{amsfonts}
\usepackage{nicefrac}
\usepackage{microtype}
\usepackage{xcolor}
\usepackage[linesnumbered,ruled,vlined]{algorithm2e}
\usepackage{booktabs,caption}
\usepackage{subcaption}
\usepackage{multirow}
\usepackage{amsmath}
\usepackage{mathtools}
\usepackage{upgreek}
\usepackage{array}
\usepackage{graphicx}
\usepackage{afterpage}
\usepackage{bm}
\usepackage{algpseudocode}
\usepackage{amsthm}
\usepackage{amssymb}
\usepackage{pifont}
\usepackage{flushend}
\usepackage{cleveref}
\usepackage{minitoc}
\newcommand{\cmark}{\ding{51}}%
\newcommand{\xmark}{\ding{55}}%
\newtheorem{assumption}{Assumption}
\newtheorem{theorem}{Theorem}
\newtheorem{lemma}{Lemma}
\newtheorem{corollary}{Corollary}
\crefname{assumption}{Assumption}{Assumptions}
\newtheorem{definition}{Definition}
\newtheorem{remark}{Remark}
\newtheorem{proposition}{Proposition}

\title{DeCAF: Decentralized Consensus-And-Factorization for Low-Rank Adaptation of Foundation Models
}

\usepackage{authblk}

\begin{document}
\author[1]{Nastaran Saadati, Zhanhong Jiang, Joshua R. Waite, Shreyan Ganguly, Aditya Balu}
\author[2]{Chinmay Hegde}
\author[1]{Soumik Sarkar}
\affil[1]{Iowa State University, Ames, IA, USA\\ \texttt{\{nsaadati, zhjiang, jrwaite, shreyang, baditya, soumiks\}@iastate.edu}}
\affil[2]{New York University, New York, NY, USA\\ \texttt{chinmay.h@nyu.edu}}

\maketitle

\begin{abstract}
Low-Rank Adaptation (LoRA) has emerged as one of the most effective, computationally tractable fine-tuning approaches for training Vision-Language Models (VLMs) and Large Language Models (LLMs).  
LoRA accomplishes this by freezing the pre-trained model weights and injecting trainable low-rank matrices, 
allowing for efficient learning of these foundation models even on edge devices. However, LoRA in decentralized settings still remains under-explored, particularly for the theoretical underpinnings due to the lack of smoothness guarantee and model consensus interference (defined formally below). 
This work improves the convergence rate of decentralized LoRA (DLoRA) to match the rate of decentralized SGD by ensuring gradient smoothness. We also introduce DeCAF, a novel algorithm integrating DLoRA with truncated singular value decomposition (TSVD)-based matrix factorization to resolve consensus interference.
Theoretical analysis shows TSVD’s approximation error is bounded and consensus differences between DLoRA and DeCAF vanish as rank increases, yielding DeCAF’s matching convergence rate.
Extensive experiments across vision/language tasks demonstrate our algorithms outperform local training and rivals federated learning under both IID and non-IID data distributions.
\end{abstract}

\section{Introduction}

\label{sec:intro}

Foundation models have gained widespread popularity, with notable examples such as CLIP~\cite{radford2021learningtransferablevisualmodels} and GPT \cite{openai2024gpt4technicalreport,10.1007/s11023-020-09548-1} showcasing exceptional performance across a wide range of tasks. However, fine-tuning foundation models for personalized downstream tasks requires immense computational resources \cite{rae2022scalinglanguagemodelsmethods, hoffmann2022trainingcomputeoptimallargelanguage, shukor2024skippingcomputationsmultimodalllms}. Low-rank adaptation (LoRA) has become a favored solution to address this challenge, and significantly lowers the number of learnable parameters during fine-tuning. LoRA achieves this by freezing the pre-trained model weights and injecting trainable low-rank matrices \cite{hu2021loralowrankadaptationlarge}. But crucially, LoRA requires access to the \emph{entire} fine-tuning dataset. 

Decentralized learning~\cite{beltran2023decentralized,hallaji2024decentralized} is a training paradigm where data is partitioned across multiple devices communicating over an arbitrary  network topology. Each device maintains its local dataset and trains a personalized model while collaborating with other devices in the network. Decentralized learning enables model training without centralized data aggregation, preserving both data locality and ensuring personalization while benefiting from collaborative knowledge sharing~\cite{el2021collaborative}. A recent work~\cite{ghiasvand2025decentralized} has attempted to study LoRA in decentralized settings and proposed Dec-LoRA to address the issue of scalable LLM fine-tuning in decentralized environments. Despite its established rigorous convergence for smooth non-convex objectives, the rate remains $\mathcal{O}(1/T^{1/3})$ (where $T$ is the number of iterations), which is worse than the standard one (i.e., $\mathcal{O}(1/\sqrt{T})$) attained for various federated or decentralized SGD methods~\cite{haddadpour2019convergence,koloskova2020unified}. 
Also, Dec-LoRA adopts the individual consensus update for low-rank matrices, leaving the so-called \textit{model consensus interference}~\cite{sun2024improving,wang2024flora,guo2024selective,zhu2024deer} issue unsolved. In the theoretical analysis, a smoothness assumption was made on the objective w.r.t full parameters. However, such an assumption fails to guarantee the smoothness of the objective w.r.t low-rank adapters~\cite{sun2024improving}.
Additionally, the empirical insights are limited due to the lack of comparison between Dec-LoRA and other distributed or decentralized counterparts, focusing only on the evaluation of language models. Thus, a question naturally arises:
    \textit{Can we develop a decentralized LoRA algorithm to address model consensus interference issue, while matching its convergence rate for non-convex and smooth objectives to $\mathcal{O}(1/\sqrt{T})$?}

\textbf{Contributions.} To this end, we first revisit DLoRA (this essentially is Dec-LoRA, and we use it for distinction) and improve its convergence rate, by imposing an assumption (Assumption~\ref{assum_3} part 1) and deriving a new smoothness constant for characterizing the convergence analysis. Our analysis shows the precise relationship between the intrinsic rank in LoRA and convergence rate in the decentralized setting.
We also develop a parameter-efficient variant of DLoRA (by freezing one of the low-rank factors, \textbf{DLoRA-FA}) to further enhance the tradeoff between performance and efficiency.
\begin{figure*}
    \centering
    \includegraphics[width=0.8\linewidth]{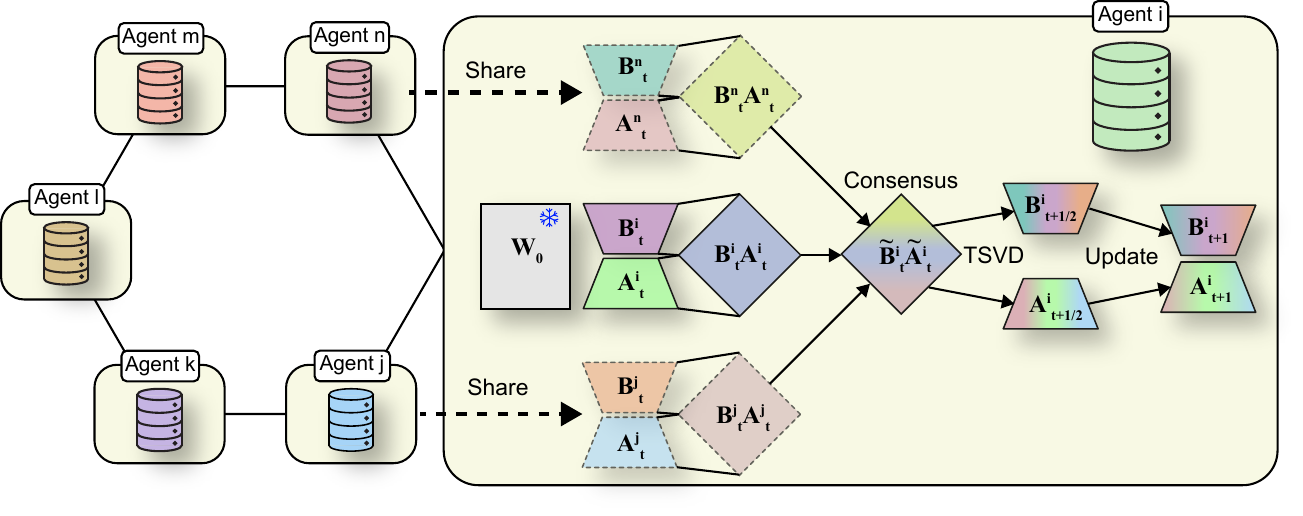}
    \caption{DeCAF with ring topology: At step $t$, agent $i$ exchanges $\mathbf{A}^i$ and $\mathbf{B}^i$ with neighbors, performs consensus on their products, applies truncated SVD, and updates using local data. Pre-trained weights $\mathbf{W}_0$ remain frozen.}
    \label{fig:dlora_diagram}
\end{figure*}
To address the challenging model consensus interference issue, we propose a novel algorithm, termed Decentralized Consensus-And-Factorization (\textbf{DeCAF}), resorting Truncated Singular Value Decomposition (TSVD)~\cite{deng2024fast} to decompose the product consensus back to individual local updates for low-rank matrices (see  \cref{fig:dlora_diagram}). We theoretically analyze the difference of consensus between DLoRA and DeCAF, surprisingly disclosing that the difference vanishes along with the number of iterations. This formally justifies even with model consensus interference, DLoRA and its federated counterparts are still being utilized in a variety of applications. Furthermore, DeCAF enjoys the same convergence rate as DLoRA.
We also provide experimental results on distinct fine-tuning tasks involving VLMs such as CLIP, and LLMs such as LLAMA2-7B ~\cite{touvron2023llama2openfoundation} to support our theoretical claims. 

To summarize our contributions:
    1) We improve the convergence rate for smooth non-convex objectives in DLoRA for from $\mathcal{O}(1/T^{1/3})$ to $\mathcal{O}(1/\sqrt{T})$ in decentralized environments and also develop a parameter-efficient variant of DLoRA (termed DLoRA-FA, see analysis in Supplementary Materials (SM)~\ref{analysis_dlora_fa}).
    2) To address the model consensus interference, we propose DeCAF, a novel variant of DLoRA, by leveraging TSVD, and theoretically study the difference of consensus between DLoRA and DeCAF to reveal the convergence rate for DeCAF. Please see \cref{tab:comparison} for more methodological comparisons.
    3) We implement our frameworks, DeCAF and DLoRA-FA, and evaluate them across various tasks involving VLMs and LLMs. Our findings demonstrate that our approaches rivals federated learning while outperforming local fine-tuning. 


\begin{table*}[ht!]
 \caption{\small{ Comparison between different methods. $N$: \# of agents, $T$: the number of iterations, $r$: the low rank satisfying $r\ll\textnormal{min}\{d,k\}$, $d$: the input dimension, $k$: the output dimension, Comm.: cost per communication round, $p$: the partial participation percentage in $(0,1]$, $N_b$:  total number of non-zero elements in the mixing matrix $\mathbf{\Pi}$ (defined in Assumption~\ref{assum_5}), Smoo.Gua.: smoothness guarantee. The full table is shown in SM.}}
  \label{tab:comparison}
  \centering
  \begin{tabular}{ccccc}
    \toprule
    Method & Rate & LLM/VLM & Comm. & Smoo.Gua.\\
    \midrule
    pFedLoRA~\cite{yi2023fedlora} & $\mathcal{O}(1/\sqrt{T})$ &  \xmark/\cmark & $\mathcal{O}((pN+1)(k+d)r)$ & \xmark\\
    FLoRA~\cite{wang2024flora} & N/A  & \cmark/\xmark & $\mathcal{O}((pN+1)(k+d)r)$ & \xmark\\
    FFA-LoRA~\cite{sun2024improving} & N/A &  \cmark/\xmark & $\mathcal{O}((pN+1)(k+d)r)$ & \cmark \\
    \hline
    Mix-of-Show~\cite{gu2024mix} & N/A &  \xmark/\cmark & $\mathcal{O}((pN+1)(k+d)r)$ & \xmark\\
    PCFT~\cite{wagner2024personalized}&N/A &  \cmark/\xmark & $\mathcal{O}(N_b(k+d)r)$ & \xmark\\
    Dec-LoRA~\cite{ghiasvand2025decentralized} & $\mathcal{O}(1/T^{1/3})$ & \cmark/\xmark & $\mathcal{O}(N_b(k+d)r)$ & \xmark\\
    \hline
    \textbf{DLoRA} & $\mathcal{O}(1/\sqrt{T})$  & \cmark/\cmark & $\mathcal{O}(N_b(k+d)r)$ &\cmark\\
    \textbf{DeCAF} & $\mathcal{O}(1/\sqrt{T})$  & \cmark/\cmark & $\mathcal{O}(N_b(k+d)r)$ &\cmark\\
    \bottomrule
  \end{tabular}
\end{table*}

\section{Related Works}
\label{sec:relatedworks}
\noindent\textbf{Low-Rank Adaptation (LoRA).}
Parameter-efficient fine-tuning (PEFT) methods, such as adapter-based or prompt-based approaches~\cite{rebuffi2017learningmultiplevisualdomains,liu2022fewshotparameterefficientfinetuningbetter,li2021prefixtuningoptimizingcontinuousprompts}, have become popular strategies for adapting pre-trained models. In this context, LoRA~\cite{hu2021loralowrankadaptationlarge} addresses the computational cost of fine-tuning by freezing the pre-trained model weights and introducing trainable low-rank matrices. 
Recent advancements in LoRA have aimed at further improving efficiency by introducing modifications to the low-rank matrices. One notable example is AFLoRA (Adaptive Freezing of Low-Rank Adaptation)~\cite{liu2024afloraadaptivefreezinglow}, which enhances PEFT by progressively freezing the low-rank matrices during fine-tuning based on a novel freezing score. To further reduce the computational demand, LoRA-FA~\cite{zhang2023lora} was developed to freeze one low-rank matrix and only train another one. Another approach is to quantize the pretrained model when backpropagating gradients into LoRA~\cite{li2023loftq,dettmers2024qlora,xu2023qa}, thus saving the memory usage and improving the sustainability. Recent works have also explored LoRA in the context of few-shot vision-language adaptation. For instance, CLIP-LoRA~\cite{zanella2024low} applies low rank adaptation for vision language models across multiple tasks by adapting both vision and language encoders. 

\noindent\textbf{Decentralized Learning.}
Most decentralized learning research has focused on small models and classification tasks~\cite{saadati2024dimat, FOTOUHI202425, jiang2017collaborative, esfandiari2021cross, assran2019stochastic}. One recent decentralized learning effort integrates LoRA for collaboratively fine-tuning LLMs~\cite{wagner2024personalizedcollaborativefinetuningondevice}. This work explores on-device, self-supervised collaborative fine-tuning of LLMs, specifically focusing on next-token prediction. While it leverages LoRA to reduce communication overhead by only exchanging LoRA weight updates, their algorithm is fundamentally different from ours—they focus on trust-based communication among agents, whereas we emphasize product-based consensus in a fully decentralized setup. Notably, their work does not provide any formal convergence guarantees. Another work~\cite{gu2024mix} proposed a method called Mix-of-Show utilizing an embedding-decomposed LoRA for each agent tuning and gradient fusion for the center node to preserve in-domain knowledge. Though they claimed the method decentralized, there was no communication among diverse agents. Essentially, their method still resembles a federated learning setup. A related effort, Dec-LoRA~\cite{ghiasvand2025decentralized}, has also explored LoRA in decentralized settings. As discussed in the introduction, despite its convergence analysis, it suffers from suboptimal rates and limitations in handling model consensus interference, and shows weak empirical performance in non-IID settings. Recent efforts have also explored decentralized LLMs in multi-robot systems~\cite{yu2024mhrc, chen2024scalable}. Another line of research is integrating federated learning (FL) with LoRA. However, FL often depends on a central server to coordinate training, which can become a bottleneck and introduce a single point of failure~\cite{9441499, bonawitz2019federatedlearningscaledesign}. This centralization may also need to address privacy concerns, as aggregated updates could still reveal sensitive information~\cite{li2020federated}. A more extensive literature survey is included in the SM~\ref{related_work_fl}.
To the best of our knowledge, no prior work offers comprehensive theoretical and empirical studies of LoRA in fully decentralized settings with rigorous convergence guarantees and strong non-IID performance.
\section{Preliminaries and Problem Formulation}
\label{sec:prelim}

\textbf{LoRA.}
The forward pass of a deep neural network involves weight matrix multiplications in numerous dense layers. Though these matrices are full-rank, when adapting the model to specific tasks, the authors in~\cite{aghajanyan2020intrinsic} show that a pre-trained LLM has a low "intrinsic dimension" which maintains the efficient learning even if it is projected into a smaller subspace. Inspired by this, Hu et al.~\cite{hu2021lora} proposed the vanilla LoRA, which implies that during fine-tuning, the updates to weights also have a low "intrinsic rank". 
Considering a pre-trained model, we denote by $\mathbf{W}_0\in\mathbb{R}^{d\times k}$ its weight matrix. The updates to $\mathbf{W}_0$ is denoted as $\Delta\mathbf{W}$ and can be represented by a low-rank decomposition between two matrices, $\mathbf{A}\in\mathbb{R}^{r\times k}$, $\mathbf{B}\in\mathbb{R}^{d\times r}$. $r$ is the rank in this context satisfies $r\ll\textnormal{min}\{d,k\}$. Then, we get:
\begin{equation}\label{eq_1}
    \mathbf{W} = \mathbf{W}_0 + \Delta\mathbf{W} = \mathbf{W}_0 + \frac{\eta}{r}\mathbf{B}\mathbf{A},
\end{equation}
where $\eta>0$ is the scaling factor for the rank stabilization. Eq.~\ref{eq_1} has been adopted in~\cite{malinovsky2024randomized} and shown more advanced capability than the vanilla LoRA in~\cite{hu2021lora}.
As LoRA only conducts fine-tuning, $\mathbf{W}_0$ is \textit{frozen} and does not get involved in the gradient update during backpropagation, while $\mathbf{A}$ and $\mathbf{B}$ matrices consist of trainable parameters. The initialization for these two matrices is critical as it should not change $\mathbf{W}$ at the beginning. Hence, each element of $\mathbf{A}_0$ is sampled from a Gaussian distribution, i.e., $a_{ij,0}\sim\mathcal{N}(0,\sigma^2)$ (where $\sigma >0$), and each element of $\mathbf{B}_0$ is set 0, i.e., $b_{ij,0}=0$.
LoRA significantly mitigates the computational complexity issue by reducing the number of trainable parameters from $\mathcal{O}(dk)$ to $\mathcal{O}((d+k)r)$.

\textbf{Problem Formulation.}
Consider a networked system involving $N$ nodes (aka agents), represented by $\mathcal{G}=(\mathcal{V},\mathcal{E})$, where $\mathcal{V}=\{1,2,...,N\}:=[N]$ is a set of nodes and $\mathcal{E}=\{(i,j),i\in\mathcal{V},j\in\mathcal{V}\}$ is a set of edges. Without loss of generality, $\mathcal{G}$ is assumed to be \textit{undirected} and \textit{connected}. 
We also denote by $Nb(i)$ the neighborhood of a node $i$ such that $Nb(i):=\{j\in\mathcal{V}|(i,j)\in\mathcal{E} \;\textnormal{or}\;i=j\}$.
In this context, each node is an LLM that is parameterized by $\mathbf{W}\in\mathbb{R}^{d\times k}$ (representing all weight matrices in multiple layers) and fine-tuned over a local dataset $\mathcal{D}_i$. 
Combining LoRA in Eq.~\ref{eq_1} with the decentralized learning formulation,
denoting by $\bm{w}$ ($\bm{w}:=(\mathbf{A},\mathbf{B})$) the low-rank adapter, we obtain the consensus optimization problem:
\begin{equation}\label{eq_3}
\textnormal{min}_{\mathbf{A}\in\mathbb{R}^{r\times k}, \mathbf{B}\in\mathbb{R}^{d\times r}} f(\bm{w},\mathbf{W}_0) = \frac{1}{N}\sum_{i=1}^N\mathbb{E}_{\xi_i\sim\mathcal{D}_i}[\mathcal{F}^i(\bm{w},\mathbf{W}_0;\xi_i)],
\end{equation}
where $f^i(\bm{w},\mathbf{W}_0):=\mathbb{E}_{\xi_i\sim\mathcal{D}_i}[\mathcal{F}^i(\bm{w},\mathbf{W}_0;\xi_i)]$ ($f^i:\mathbb{R}^{(d+k)r}\times\mathbb{R}^{d\times k}\to\mathbb{R}$) are smooth non-convex functions.
One scrutiny is that in Eq.~\ref{eq_3} pre-trained model parameters are frozen during fine-tuning. However, it may vary slightly different if some more parameters are further frozen in $\bm{w}$. 
Another implication is that we do not search for the optimal product of low-rank matrices, i.e., $(\mathbf{BA})^*$. Instead, we are interested in $\mathbf{B}^*\mathbf{A}^*$ due to the separate updates for them. Therefore, throughout the analysis, we define the optimal objective loss $f^*:=f(\bm{w}^*,\mathbf{W}_0)>-\infty$, where $\bm{w}^*=(\mathbf{A}^*,\mathbf{B}^*)=\textnormal{argmin}_{\mathbf{A}\in\mathbb{R}^{r\times k},\mathbf{B}\in\mathbb{R}^{d\times r}}f(\bm{w},\mathbf{W}_0)$.

\section{Main Results}\label{main_results}
Denote by $\bm{g}^{i}=\frac{1}{|\mathcal{S}_i|}\sum_{s\in\mathcal{S}_i}\nabla f^i_s(\bm{w}^i,\mathbf{W}_0)$ ($\bm{w}^i=(\mathbf{A}^i,\mathbf{B}^i$)) the stochastic gradient for the $i$-th node, where $\mathcal{S}_i$ is a mini-batch randomly sampled from $\mathcal{D}_i$. $\bm{g}^i$ is assumed to be an unbiased estimate of $\nabla f^i(\bm{w}^i,\mathbf{W}_0)$, i.e., $\nabla f^i(\bm{w}^i,\mathbf{W}_0)=\mathbb{E}[\bm{g}^i]$. We defer all proofs into SM~\ref{additional_analysis}.

\textbf{DLoRA.}
Algorithm~\ref{alg:dlora} shows the algorithmic framework for fine-tuning the LoRA matrices $\mathbf{A}$ and $\mathbf{B}$ using SGD (extension to momentum-based SGD and Adam is included in SM). 
\begin{algorithm}
  \caption{\textsc{DLoRA} (\textcolor{blue}{DeCAF})}
  \label{alg:dlora}
  \SetKwInOut{Input}{Input}
  \SetKwInOut{Output}{Output}
  \Input{mixing matrix $\mathbf{\Pi}=[\pi_{ij}]_{i,j\in N}$, the \# of iterations $T$, initialization of $\mathbf{A}_1^i, \mathbf{B}_1^i, \forall i\in\mathcal{V}$, step size $\alpha$, $\mathbf{W}_0$, $\mathcal{D}_i, i\in\mathcal{V}$, communication frequency $\tau$, TSVD operator $\mathcal{T}(\cdot)$}
  \Output{$\{\mathbf{A}^i_T, \mathbf{B}^i_T\}_{i=1}^N$}  
  \BlankLine
  \For{ $t$ in $1:T$ }
  { 
    \For{each agent $i\in\mathcal{V}$}
    { 
    Calculate the stochastic gradients $\bm{g}^{i,\mathbf{\bullet}}_t=\frac{1}{|\mathcal{S}_i|}\sum_{s\in\mathcal{S}_i}\nabla_\mathbf{\bullet} f^i_s(\bm{w}^i_t,\mathbf{W}_0), \bullet\in\{\mathbf{A},\mathbf{B}\}$\;
    \eIf{$t$ mod $\tau$=0}
    {     Broadcast the low-rank matrices $\mathbf{A}^i_t, \mathbf{B}^i_t$ to and receive $\mathbf{A}^j_t, \mathbf{B}^j_t$ from nodes in $Nb(i)$\;
    \# Individual consensus in DLoRA\;
    $\mathbf{A}^i_{t+1/2}=\sum_{j\in Nb(i)}\pi_{ij}\mathbf{A}^j_{t}$\;
    $\mathbf{B}^i_{t+1/2}=\sum_{j\in Nb(i)}\pi_{ij}\mathbf{B}^j_{t}$\;
    \textcolor{blue}{
    \# Product consensus in DeCAF\;
    $\tilde{\mathbf{B}}^i_t\tilde{\mathbf{A}}^i_t=\sum_{j\in Nb(i)}\pi_{ij}\mathbf{B}^j_t\mathbf{A}^j_t$\;
    $\mathbf{A}^i_{t+1/2}, \mathbf{B}^i_{t+1/2} = \mathcal{T}(\tilde{\mathbf{B}}^i_t\tilde{\mathbf{A}}^i_t)$\;
    }
    }{$\mathbf{A}^i_{t+1/2}=\mathbf{A}^i_{t}$\;
    $\mathbf{B}^i_{t+1/2}=\mathbf{B}^i_{t}$\;
    }
    $\mathbf{A}^i_{t+1}=\mathbf{A}^i_{t+1/2}-\alpha \bm{g}^{i,\mathbf{A}}_t$\;
    $\mathbf{B}^i_{t+1}=\mathbf{B}^i_{t+1/2}-\alpha \bm{g}^{i,\mathbf{B}}_t$\;    }
  }
\end{algorithm}
Line 3 in Algorithm~\ref{alg:dlora} manifests the mini-batch stochastic gradient w.r.t. both low-rank matrices $\mathbf{A}^i$ and $\mathbf{B}^i$ only. $\mathbf{W}_0$ is still necessary to fulfill the loss value calculation in practice. Line 4 implies that the communication step can be implemented periodically, therefore reducing the number of communication rounds (In this work, we study the worst case $\tau=1$).
Line 5 is unique in LoRA as the communication is $\mathbf{A}$ and $\mathbf{B}$ instead of $\mathbf{BA}$, thus reducing the overhead per communication round from $\mathcal{O}(kd)$ to $\mathcal{O}((k+d)r)$. 
Lines 6-8 show the individual consensus for both $\mathbf{A}$ and $\mathbf{B}$, resulting in model consensus interference defined in Definition~\ref{definition_1}.
To further mitigate the challenging model consensus interference issue, in Line 10, DeCAF conducts consensus in product by taking parameter changes $\Delta\mathbf{W}^j_t$. However, since the local updates to $\mathbf{A}^i_t$ and $\mathbf{B}^i_t$ are separate, a decomposition is required in this context. We resort to a popular low-rank approximation technique (TSVD) and will show its induced approximation error.
Another popular scheme is to freeze $\mathbf{A}$ matrix, which further enables the communication reduction from $\mathcal{O}((k+d)r)$ to $\mathcal{O}(dr)$, but possibly with the compromise of model performance. To show the tradeoff between the performance and efficiency, we also develop a parameter-efficient variant of DLoRA by freezing $\mathbf{A}$ matrix (DLoRA-FA detailed in Algorithm~\ref{alg:dlora_fa}) and provide theoretical analysis in SM~\ref{analysis_dlora_fa} for completeness. Lines 15-16 are local updates for both $\mathbf{A}^i$ and $\mathbf{B}^i$, reducing the computational cost from $\mathcal{O}(dk)$ to $\mathcal{O}((d+k)r)$.
Existing analysis has achieved the convergence rate of $\mathcal{O}(1/T^{1/3})$, whereas there is a gap between this result and the best available rate for regular decentralized SGD~\cite{saadati2024dimat}, $\mathcal{O}(1/\sqrt{T})$. Additionally, empirical results in LoRA have revealed a clear relationship between the rank $r$ and accuracy, but the convergence error bound in~\cite{ghiasvand2025decentralized} fails to show this explicitly.
Thereby, we will also study how the rank $r$ impacts the convergence error bound. 
\begin{assumption}\label{assum_1}
    There exists a constant $L>0$ such that $\|\nabla f^i(\mathbf{W})-\nabla f^i(\mathbf{W}')\|_F\leq L\|\mathbf{W}-\mathbf{W}'\|_F$, for all $\mathbf{W},\mathbf{W}'\in\mathbb{R}^{d\times k}$ and $i\in\mathcal{V}$. $\|\cdot\|_F$ is the Frobenius norm.
\end{assumption}
This assumption implies that $f^i(\mathbf{W})$ is $L$-smooth without low-rank decomposition and has generically been applied in numerous existing works~\cite{zeng2018nonconvex,ge2023gradient,francis2023decentralized}. Nevertheless, a recent work~\cite{sun2024improving} has reported that \textit{with low-rank adaptation, even if $f^i(\mathbf{W})$ is smooth, there is no guarantee for the smoothness of $f^i(\bm{w},\mathbf{W}_0)$}. We include a counter-example for this with detailed proof in SM~\ref{smooth_issue}. 
To circumvent this,
we derive a new smoothness constant for DLoRA. Before presenting a key lemma to characterize the new smoothness constant, we impose two assumptions.
\begin{assumption}\label{assum_2}
    For any $\mathbf{A}^i$ and $\mathbf{B}^i$, for all $i\in\mathcal{V}$ induced by Algorithm~\ref{alg:dlora}, their largest singular values i.e., $\sigma_1(\mathbf{A}^i),\sigma_1(\mathbf{B}^i)$, satisfy the condition:
$ 
        0<\textnormal{max}\{\sigma_1(\mathbf{A}^i),\sigma_1(\mathbf{B}^i)\}\leq c,
$
where $c>0$.
\end{assumption}
Assumption~\ref{assum_2} empirically makes sense as the optimization problem on $\mathbf{A}^i$ and $\mathbf{B}^i$ should be well-posed such that the optimal solution can be reached. Thus, both low-rank matrices naturally satisfy the condition in the assumption. We present another assumption on any two parameters $\mathbf{W}$ and $\mathbf{W}'$.
\begin{assumption}\label{assum_3}
    1) There exists a constant $C>0$ such that for any parameters $\mathbf{W}$, $\mathbf{W}'$, $\mathbf{A}$, $\mathbf{A}'$, $\mathbf{B}$, $\mathbf{B}'$, $\|\mathbf{W}-\mathbf{W}'\|_F\leq C(\|\mathbf{A}-\mathbf{A}'\|_F+\|\mathbf{B}-\mathbf{B}'\|_F)$; 2) $f^i$ is Lipschitz continuous at $\mathbf{W}\in\mathbb{R}^{d\times k}$, i.e., there exists a constant $G>0$, $\|f^i(\mathbf{W}')-f^i(\mathbf{W})\|_F\leq G\|\mathbf{W}'-\mathbf{W}\|_F$ for all $\mathbf{W}', \mathbf{W}\in\mathbb{R}^{d\times k}$.
\end{assumption}
We leverage two instances for justification of Part 1). They are $\mathbf{A}=\mathbf{A}'$ and $\mathbf{B}=\mathbf{B}'$, respectively. In these two scenarios, either $\mathbf{A}$ or $\mathbf{B}$ is frozen. Then $C$ can be quickly determined as $\frac{\eta c}{\sqrt{r}}$, as for any matrix $\mathbf{Z}\in\mathbb{R}^{r\times k} \;\text{or}\;\mathbb{R}^{d\times r}$, $\|\mathbf{Z}\|_F=\sqrt{\sum_l\sigma_l^2(\mathbf{Z})}\leq \sqrt{r}\sigma_1(\mathbf{Z})$ and the fact that $\|\mathbf{W}-\mathbf{W}'\|_F=\frac{\eta}{r}\|\mathbf{BA}-\mathbf{B'A'}\|_F$. In a generalized scenario, one can set a sufficiently large constant $C$ to ensure the condition to hold, though it leads to a looser bound.
Part 2 in Assumption~\ref{assum_3} is also generic in decentralized learning~\cite{zhou2020bypassing,jiang2017collaborative}. Furthermore, this assumption does not apply to the gradient of $f^i$ w.r.t. low-rank matrix $\mathbf{A}$ or $\mathbf{B}$, primarily helping characterize the new smoothness constant.
\begin{lemma}\label{lemma_1}
    Let Assumptions~\ref{assum_1},~\ref{assum_2},~\ref{assum_3} hold. Suppose that the parameter for an agent $\mathbf{W}$ satisfies LoRA, i.e., $\mathbf{W}=\mathbf{W}_0+\frac{\eta}{r}\mathbf{B}\mathbf{A}$. Then, for any given $\bm{w}, \bm{w}'$, we have the following relationship:
$
        \|\nabla f^i(\bm{w},\mathbf{W}_0)-\nabla f^i(\bm{w}',\mathbf{W}_0)\|_F\leq \hat{L}\|\bm{w}-\bm{w}'\|_F, \hat{L}=\frac{\eta(2LC\sqrt{r}c+G)}{r}.
$
\end{lemma}
\begin{assumption}\label{assum_4} There exist $\zeta, \kappa > 0$ such that
    a) The variance of stochastic gradient in each node is bounded: $\mathbb{E}_{\xi\sim\mathcal{D}_i}[\|\nabla \mathcal{F}^i(\bm{w},\mathbf{W}_0;\xi)-\nabla f^i(\bm{w},\mathbf{W}_0)\|^2_F]\leq \zeta^2, i=1,2,...,N$; b) The gradient diversity is uniformly bounded, i.e., $\frac{1}{N}\sum_{i=1}^N\|\nabla f^i(\bm{w},\mathbf{W}_0)-\nabla f(\bm{w},\mathbf{W}_0)\|^2_F\leq \kappa^2, \forall \bm{w}, i=[N]$.
\end{assumption}
The above assumptions have been used in~\cite{li2019convergence,zhang2012communication,stich2018local,stich2018sparsified,yu2019parallel},
signifying the bounded noise for local agents and quantifying the gradient diversity among agents, due to the different data distributions. 
The next assumption for matrix $\mathbf{\Pi}$ has been utilized frequently in existing works~\cite{esfandiari2021cross,jiang2017collaborative,saadati2024dimat}. 
\begin{assumption}\label{assum_5} $\mathbf{\Pi}\in\mathbb{R}^{N\times N}$ is a symmetric doubly stochastic matrix satisfying $\lambda_1(\mathbf{\Pi})=1$ and
   $
       \textnormal{max}\{|\lambda_2(\mathbf{\Pi})|, |\lambda_N(\mathbf{\Pi})|\}\leq \sqrt{\rho}<1,
   $
where $0<\rho<1$, $\lambda_l(\cdot)$ is the $l$-th largest eigenvalue of $\mathbf{\Pi}$.
\end{assumption}
To comply with the tradition of optimization and ease the analysis for convergence, we use vectors for each matrix-wise variable. Without loss of generality, we denote by $\bm{\theta}_0\in\mathbb{R}^{dk}$ and $\bm{v}\in\mathbb{R}^{n} (n=(d+k)r)$ the frozen model parameters and the low-rank adapter. 
In practice, this can be accomplished through the operation of flattening. We will now present the first result in the following.
\begin{theorem}\label{dlora-theo}
    Let all assumptions hold. If the step size $\alpha\leq \frac{1-\sqrt{\rho}}{4\sqrt{2}\hat{L}}$ in Algorithm~\ref{alg:dlora}, then for all $T\geq 1$, the relationship holds true:
$
        \frac{1}{T}\sum_{t=1}^T\mathbb{E}[\|\nabla f(\bar{\bm{v}}_t,\bm{\theta}_0)\|^2]\leq \frac{2D}{\alpha T}+\frac{\hat{L}\alpha\zeta^2}{N}+\frac{\hat{L}^2\sum_{i=1}^N\|\bm{v}_0^i\|^2}{TN(1-\rho)}+\frac{16\alpha^2\kappa^2\hat{L}^2}{(1-\sqrt{\rho})^2}+\frac{4\alpha^2\zeta^2\hat{L}^2}{1-\rho},
$
where $D=f(\bar{\bm{v}}_0,\bm{\theta}_0)-f^*$, $\hat{L}=\frac{\eta(2LC\sqrt{r}c+G)}{r}$, $\bar{\bm{v}}_t=\frac{1}{N}\sum_{i=1}^N\bm{v}^i_t$, $\|\cdot\|$ is Euclidean norm.
\end{theorem}

It is immediately observed that the error bound in Theorem~\ref{dlora-theo} is primarily attributed to the network error caused by local gradient variance and global gradient diversity. The term related to $\sum_{i=1}^N\|\bm{v}^i_0\|^2$ is due to the non-zero initialization of matrix $\mathbf{A}$.
When the topology is \textit{dense}, like a fully connected network, where $\rho$ is smaller, the error bound is smaller. On the contrary, if the topology is \textit{sparse}, such as a ring, the error bound is larger. 
We also include the analysis for the regularly centralized setting in
SM~\ref{centralized_analysis}. 
In a non-asymptotic convergence point of view, DLoRA is supposed to converge to a desired accuracy after a certain number of epochs, which is shown as below.
\begin{corollary}\label{corollary_1}
    Let $\alpha\lesssim\sqrt{\frac{N}{T}}$. With Theorem~\ref{dlora-theo}, then for all $T\geq\frac{32N\hat{L}^2}{(1-\sqrt{\rho})^2}$, we have
$    
        \frac{1}{T}\sum_{t=1}^T\mathbb{E}[\|\nabla f(\bar{\bm{v}}_t,\bm{\theta}_0)\|^2]\lesssim\sqrt{\frac{1}{NT}}+\sqrt{\frac{1}{NTr}}+\frac{N}{(1-\rho)Tr}+\frac{N}{(1-\sqrt{\rho})^2Tr}.
$ $\lesssim$ is the shorthand for standard big-oh notation, i.e., $a=\mathcal{O}(b)$.
\end{corollary}

\noindent\textbf{Convergence rate.} To the best of our knowledge, this is the first result to show how the convergence rate evolves with the ``intrinsic rank" $r$ in the decentralized setting. When $r$ is larger to allow for more learnable parameters, the convergence error is smaller. This intuitively makes sense as a larger number of parameters make the model more expressive.
The dominated term on the right hand side of the last inequality is 
$\mathcal{O}(\sqrt{1/(NT)}+\sqrt{1/(NTr)})$ when $T$ is sufficiently large, implying the \textit{linear speedup} even with LoRA. Also, at the early phase of optimization, the larger rank suggests the faster convergence behavior. However, when the convergence is close to the optimal regime, $\mathcal{O}(\sqrt{1/(NT)})$ will stand out so that even distinct rank values result in close model performance. This phenomenon is validated correctly by empirical evidence of VLMs.
Compared to~\cite{ghiasvand2025decentralized}, our work improves the convergence rate from $\mathcal{O}(1/T^{1/3})$ to $\mathcal{O}(1/T^{1/2})$ and shows explicit dependency of error bound on $r$.
\textbf{DeCAF.}
We will now study the impact of TSVD and quantify the concrete difference between individual consensus in DLoRA and product consensus in DeCAF. We first define formally the model consensus interference in the following.
\begin{definition}\label{definition_1} (Model Consensus Interference)
    Suppose that the parameter of agent $i$, $\mathbf{W}^i$ satisfies LoRA, i.e., $\mathbf{W}^i=\mathbf{W}_0+\frac{\eta}{r}\mathbf{B}^i\mathbf{A}^i$. In the decentralized learning setting, the product of individual consensus for $\mathbf{A}^i$ and $\mathbf{B}^i$ is not equal to the consensus of product, i.e., 
$
        \sum_{j\in Nb(i)}\pi_{ij}\mathbf{B}^j\mathbf{A}^j \neq \sum_{j\in Nb(i)}\pi_{ij}\mathbf{B}^j\sum_{j\in Nb(i)}\pi_{ij}\mathbf{A}^j.
$
\end{definition}
In SM~\ref{additional_analysis}, we detail how model consensus interference stems from the update of LoRA. Intuitively, $\mathbf{B}^i\mathbf{A}^i$ represents the full knowledge of parameter changes to the pretrained model $\mathbf{W}_0$ of agent $i$ such that $\sum_{j\in Nb(i)}\pi_{ij}\mathbf{B}^j\mathbf{A}^j$ precisely indicates the collaborative knowledge sharing from other agents in the neighborhood of agent $i$. 
Freezing $\mathbf{A}^j$ matrix on both sides can alternatively resolve the issue (cf. Theorem~\ref{model_interference_theo} in SM~\ref{analysis_dlora_fa}), but at the cost of performance drop, which is theoretically justified in Corollary~\ref{corollary_2} (in SM~\ref{analysis_dlora_fa}) and evidently validated in numerical results. In DeCAF, we keep the consensus of product and use a TSVD operator to acquire individual $\mathbf{A}^i$ and $\mathbf{B}^i$. As Lines 10 and 11 in Algorithm~\ref{alg:dlora} indicate, the TSVD operator adopts $\mathcal{T} (\cdot):=\mathbf{U}_r\Sigma_r\mathbf{V}^\top_r$ and produces the output matrices $\mathbf{A}^i_{t+1/2}$ and $\mathbf{B}^i_{t+1/2}$, where $\mathbf{A}^i_{t+1/2} = (\mathbf{V}_r\Sigma^{1/2})^\top$ and $\mathbf{B}^i_{t+1/2} = \mathbf{U}_r\Sigma_r^{1/2}$. Alternatively, one can equal $\mathbf{A}^i_{t+1/2}$ and $\mathbf{B}^i_{t+1/2}$ to $\mathbf{V}_r^\top$ and $\mathbf{U}_r\Sigma_r$, which yields analogous performance in practice.
However, one may wonder if the TSVD step results in approximation error and what the upper bound is for this. We now state the result in the sequel.
\begin{proposition}\label{prop_1}
    Let Assumption~\ref{assum_2} hold. Suppose that the parameter of agent $i$, $\mathbf{W}^i$ satisfies LoRA, i.e., $\mathbf{W}^i = \mathbf{W}_0+\frac{\eta}{r}\mathbf{B}^i\mathbf{A}^i$ and that the TSVD operator satisfies $\mathcal{T}(\tilde{\mathbf{B}}^i\tilde{\mathbf{A}}^i):=\mathbf{U}_r\Sigma_r\mathbf{V}_r^\top$, where 
    $\tilde{\mathbf{B}}^i\tilde{\mathbf{A}}^i=\sum_{j\in Nb(i)}\pi_{ij}\mathbf{B}^j\mathbf{A}^j$, $\mathbf{U}_r\in\mathbb{R}^{d\times r}$, $\Sigma\in\mathbb{R}^{r\times r}$, $\mathbf{V}_r\in\mathbb{R}^{k\times r}$. Then, in DeCAF, the approximation error between $\tilde{\mathbf{B}}^i_t\tilde{\mathbf{A}}^i_t$ and $\mathcal{T}(\tilde{\mathbf{B}}^i_t\tilde{\mathbf{A}}^i_t)$ satisfies the relationship $\|\tilde{\mathbf{B}}^i_t\tilde{\mathbf{A}}^i_t-\mathbf{B}^i_{t+1/2}\mathbf{A}^i_{t+1/2}\|_F\leq \sqrt{(|Nb(i)|-1)r}c^2$, for any $t\geq 0$.
\end{proposition}
We remark on the memory and run time required by TSVD. For memory, since the rank in this context has been fixed as $r\ll\text{min}\{d,k\}$, storing $\mathbf{U}_r, \Sigma_r, \mathbf{V}_r$ takes $\mathcal{O}((d+k)r)$ space, matching the computational and communication cost. For run time, it is roughly $\mathcal{O}(rdk)$ per iteration and $\mathcal{O}(Trdk/\tau)$ during $T$ iterations, where $\tau$ assists in reducing run time for TSVD. Another interesting observation from Proposition~\ref{prop_1} is that the approximation error is related to the size of neighborhood of agent $i$, indicating that if more agents are involved, TSVD may lead to larger error. This is empirically supported in our results. The extreme case $|Nb(i)|=1$ implying no approximation error due to no connections contradicts the assumption of a connected graph such that the minimum for the approximation error is $\sqrt{r}c^2$.
While model consensus interference remains in DLoRA, numerous existing works~\cite{ghiasvand2025decentralized,ye2024openfedllm,zhang2024towards} have overlooked this issue and showed appealing empirical results. While the primary distinction between DLoRA and DeCAF lies in the use of consensus, DLoRA benefits from the best-known sublinear convergence rate. We then wonder if the difference in consensus can be quantified and also whether DeCAF enjoys the same convergence rate. 

\begin{table*}[htbp]
    \centering
    \caption{\small{Test accuracy on IID and non-IID datasets using different algorithms with CLIP and LoRA (rank 2, 10 agents, 500 epochs, fully connected topology). Best and second-best are bolded and underlined.}}

    \label{tab:accuracy_results_updated}
    \begin{tabular}{lcccccc}
        \toprule
        \multirow{2}{*}{\textbf{Algorithm}} & \multicolumn{3}{c}{\textbf{IID}} & \multicolumn{3}{c}{\textbf{Non-IID}} \\
        \cmidrule(lr){2-4} \cmidrule(lr){5-7}
        & \textbf{Flowers} & \textbf{UCF} & \textbf{Food101} & \textbf{Flowers} & \textbf{UCF} & \textbf{Food101} \\
        \midrule
        \textbf{Centralized}~\cite{zanella2024low} & 98.0 & 86.7 & 84.2 & 98.0 & 86.7 & 84.2 \\
        \hline
        \textbf{Local} 
            & $93.20$ & $78.06$ & $78.96$ 
            & $39.69$ & $38.15$  & $46.41$ \\
         \textbf{Local-FA} 
            & $92.96$ & $81.10$ & $85.36$  
            & $43.79$ & $41.71$ & $49.87$ \\
        \textbf{FedAvg} 
        & \boldmath$98.69$ & \boldmath{$89.02$} & $86.63$ & \boldmath$96.41$ & \boldmath$84.78$ & $83.83$ \\
         \textbf{FedAvg-FA} 
         & $96.36$ & $84.87$ & \underline{$88.06$} 
         & $93.04$ & $82.53$ & \underline{$87.80$} \\
        \textbf{DLoRA} 
            & \boldmath{$98.69$} & \underline{$88.77$} & $86.58$
            & $94.92$ & $77.35$ & $81.89$ \\
        \textbf{DLoRA-FA (ours)} 
            & $96.23$ & $84.69$ & \underline{$88.06$} 
            & $92.47$ & $82.25$ & \boldmath{$87.90$} \\
        \textbf{DeCAF (ours)} 
            & \underline{$97.50$} & $86.50$ & \boldmath{$88.08$} & \underline{$95.33$} 
            & \underline{$84.31$} & \underline{$87.80$} \\
        \bottomrule
    \end{tabular}
    \label{tab:evaluation_vlm_all_updated}
\end{table*} 
\begin{table*}[htbp]
    \centering
    \caption{\small{F1 scores on WIC and BoolQ using LLAMA2-7B, trained with 10 agents over 150 batches on a Fully Connected topology. CNT denotes centralized training. Best and second-best scores are in bold and underlined.}}
    \label{tab:evaluation_metrics}
    \begin{tabular}{lc|ccccccc}
        \toprule
        \textbf{Data} & \textbf{CNT} &
        \textbf{Local} & \textbf{Local-FA} & \textbf{FedAvg} 
        & \textbf{FedAvg-FA} & \textbf{DLoRA} & \textbf{DLoRA-FA} 
        & \textbf{DeCAF} \\
        \midrule
        \textbf{WIC}   & $63.0$  & $62.5$  &  $63.9$ & \underline{$64.4$}  & \underline{$64.4$}  & \boldmath{$64.6$}          & \underline{$64.4$}         & \underline{$64.4$} \\
        \textbf{Boolq}   & $64$   & \underline{$62.0$}   & $58.0$ & \boldmath{$71.5$} & $60.8$  & \underline{$62.0$} & $60.0$ & \underline{$62.0$}  \\
        \bottomrule
    \end{tabular}
\end{table*}

\begin{theorem}\label{theorem_2}
    Denote by $\mathcal{E}_T=\mathbb{E}[\|\sum_{j\in Nb(i)}\pi_{ij}\mathbf{B}^j_T\mathbf{A}^j_T - \sum_{j\in Nb(i)}\pi_{ij}\mathbf{B}^j_T\sum_{j\in Nb(i)}\pi_{ij}\mathbf{A}^j_T\|_F]$ the consensus error between DLoRA and DeCAF at any $T>0$ such that the following relationship holds
    $
        \mathcal{E}_T\leq \frac{2\alpha G\eta c}{\sqrt{r}(1-\sqrt{\rho})} + (\rho^{\frac{T}{2}}c+\frac{\alpha\eta cG}{r(1-\sqrt{\rho})})\frac{\alpha \eta cG}{r(1-\sqrt{\rho})}.
    $
\end{theorem}
Theorem~\ref{theorem_2} signifies the upper bound for the difference between product of individual consensus and consensus of product for the first time, showing the impact of rank $r$, the spectral gap $1-\sqrt{\rho}$, and the step size $\alpha$ on the bound. When $T\to\infty$, the term $\rho^{T/2}$ vanishes, exhibiting that the error reduces if $r$ increases, which is evidently validated. We have known that when $r$ is large, $\mathbf{B}\mathbf{A}$ is more expressive in the parameter change to the frozen parameters $\mathbf{W}_0$. 
This intuitively implies that DLoRA can still be validly used to approximate DeCAF if setting a proper rank. Additionally, we are also interested in the decaying rate of $\mathcal{E}_T$ as it affects the ultimate convergence rate of DeCAF. Suppose that $\alpha$ follows the condition in Theorem~\ref{dlora-theo}, i.e, $\alpha\lesssim \sqrt{N/T}$. We have 
$
    \mathcal{E}_T\lesssim\frac{\sqrt{N}}{\sqrt{Tr}(1-\sqrt{\rho})}+\frac{\sqrt{N}\rho^{T/2}}{\sqrt{T}r(1-\sqrt{\rho})}+\frac{N}{Tr^2(1-\sqrt{\rho})^2}.
$
Since $0<\rho<1$, when $T$ is sufficiently large, the second term on the right hand side of the above inequality vanishes, and the first term dominates, leading to the rate of $\mathcal{O}(1/\sqrt{T})$. Combining the convergence rate of DLoRA yields that the convergence rate of DeCAF still remains $\mathcal{O}(1/\sqrt{T})$.

\begin{figure*}[htp]
    \centering
    \begin{subfigure}[b]{0.32\textwidth}
        \centering
        \includegraphics[width=\textwidth]{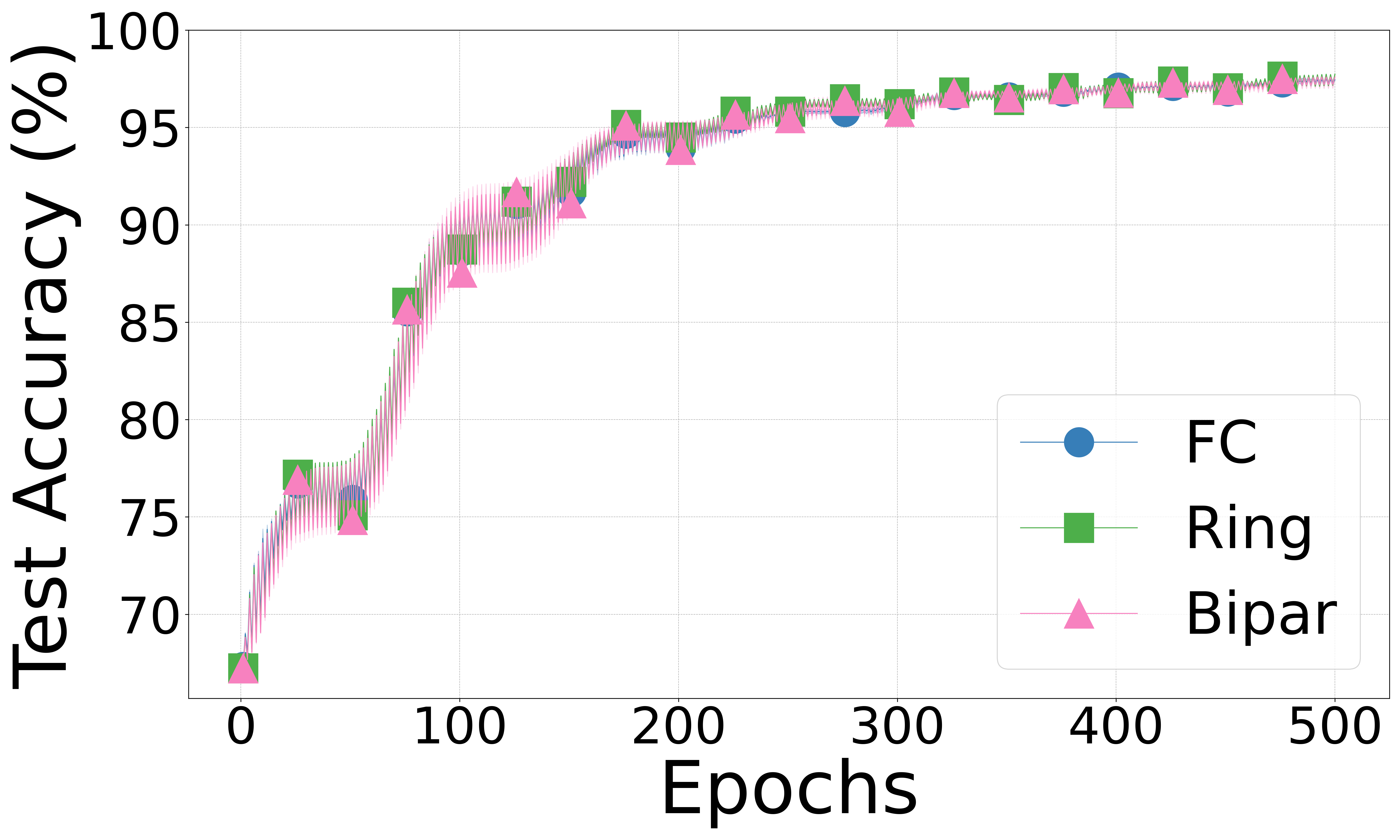}
        \caption{10-Agents, IID}
        \label{fig:topology_iid}
    \end{subfigure}
    \begin{subfigure}[b]{0.32\textwidth}
        \centering
        \includegraphics[width=\textwidth]{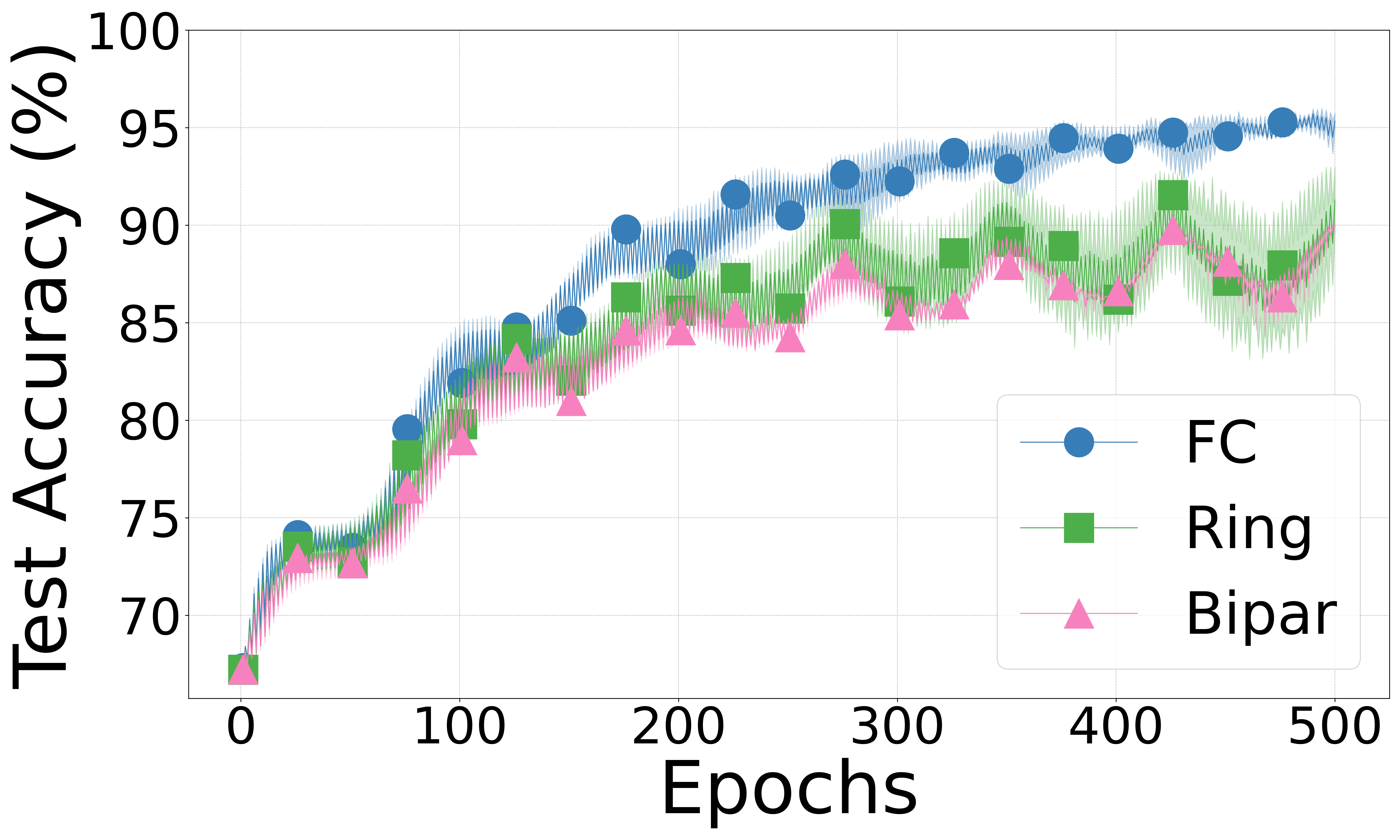}
        \caption{10-Agents, Non-IID}
        \label{fig:topology_non_iid}
    \end{subfigure}
    \begin{subfigure}[b]{0.32\textwidth}
        \centering
        \includegraphics[width=\textwidth]{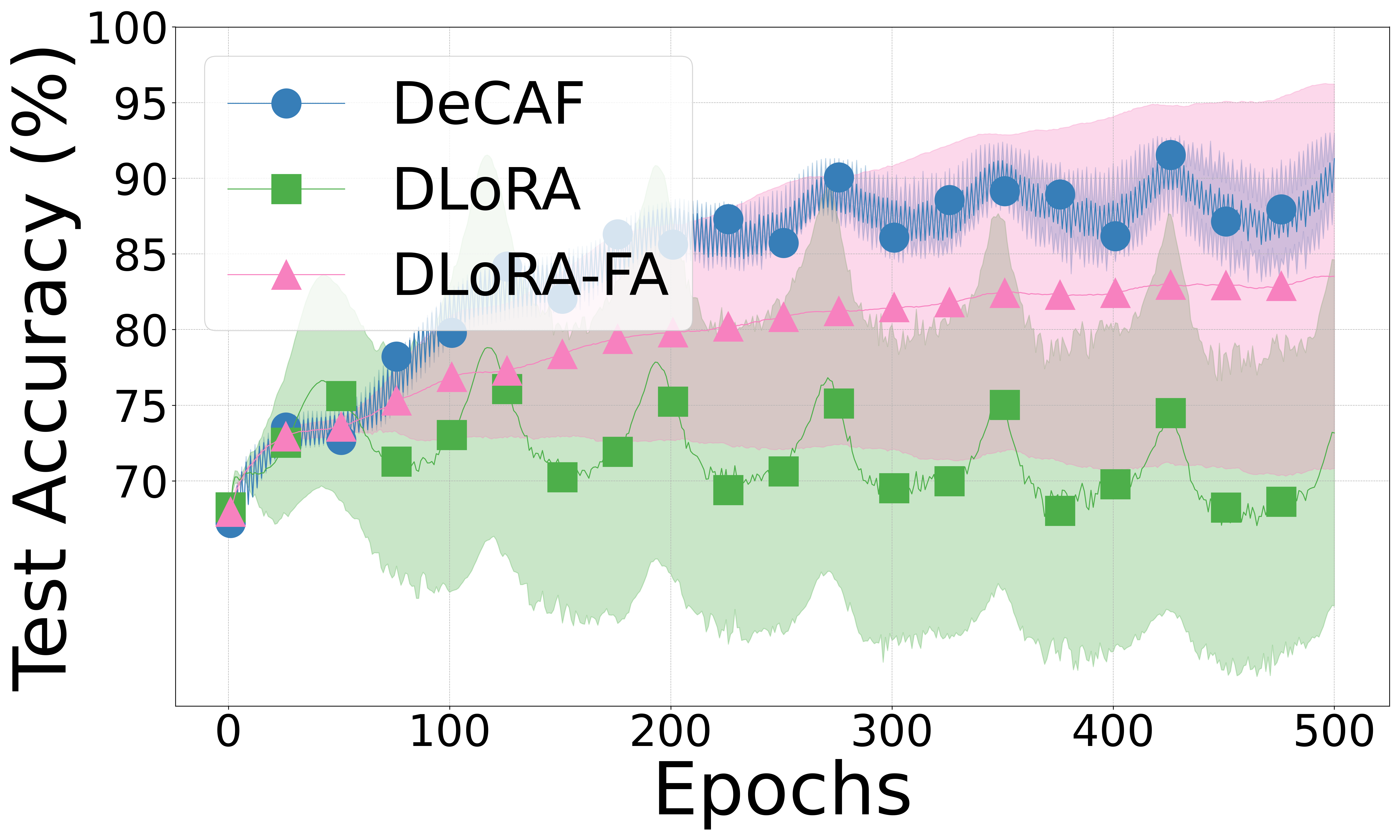}
        \caption{Algorithm (Ring, Non-IID)}
        \label{fig:algorithm}
    \end{subfigure}
    
   \caption{Test accuracy on the \textbf{Flowers} dataset with \textbf{CLIP}: (\ref{fig:topology_iid}) IID and (\ref{fig:topology_non_iid}) non-IID topologies using DeCAF; (\ref{fig:algorithm}) decentralized algorithms on non-IID ring topology (40 shots/class, rank 2).}
    \label{fig:ablation_study}
\end{figure*}
\begin{figure*}[htp]
    \centering
    \begin{subfigure}[b]{0.32\textwidth}
        \centering
        \includegraphics[width=\textwidth]{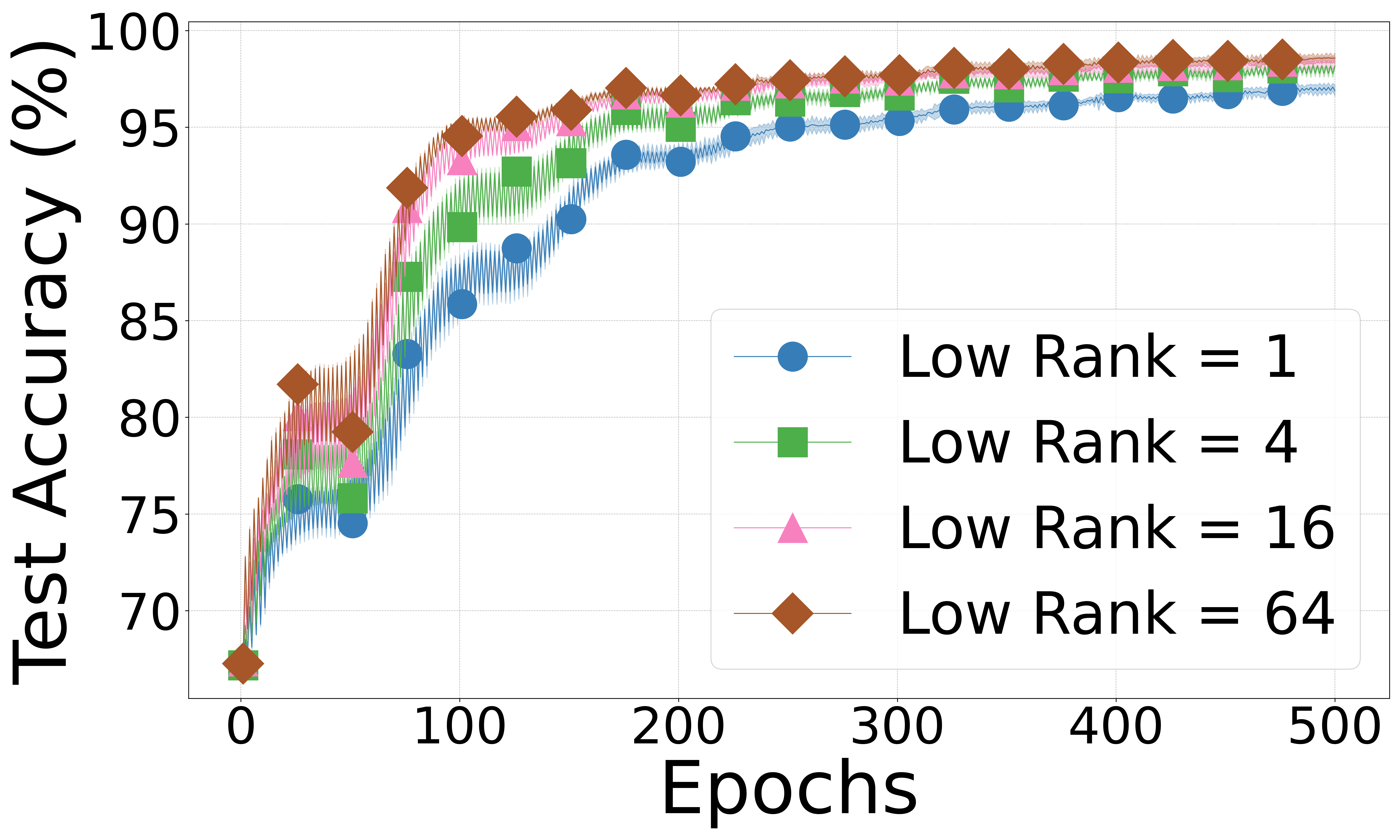}
        \caption{Low-Rank Variation (DeCAF)}
        \label{fig:low_rank}
    \end{subfigure}
    \begin{subfigure}[b]{0.32\textwidth}
        \centering
        \includegraphics[width=\textwidth]{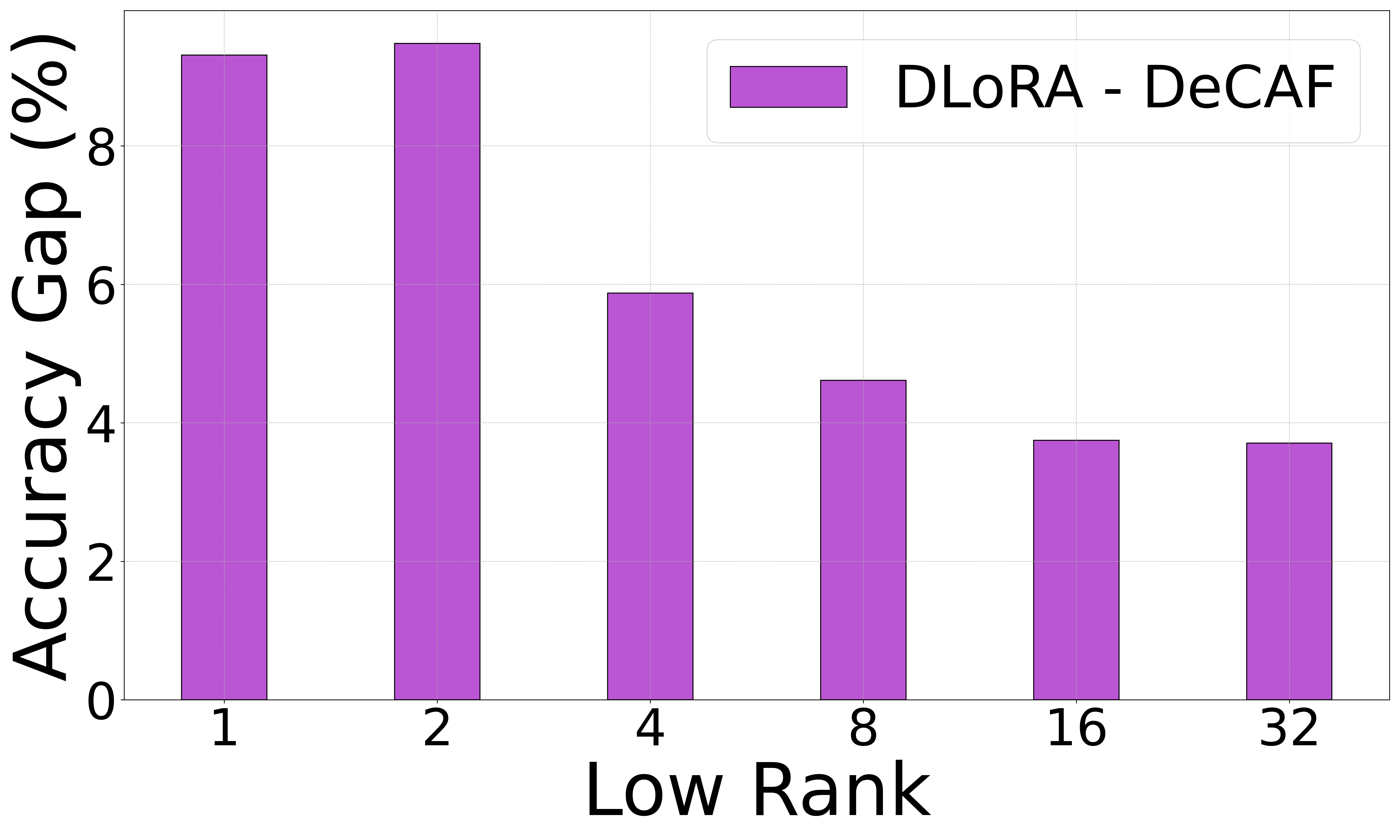}
        \caption{Accuracy gap}
        \label{fig:diff_acc_dl}
    \end{subfigure}
   \begin{subfigure}[b]{0.32\textwidth}
        \centering
        \includegraphics[width=\textwidth]{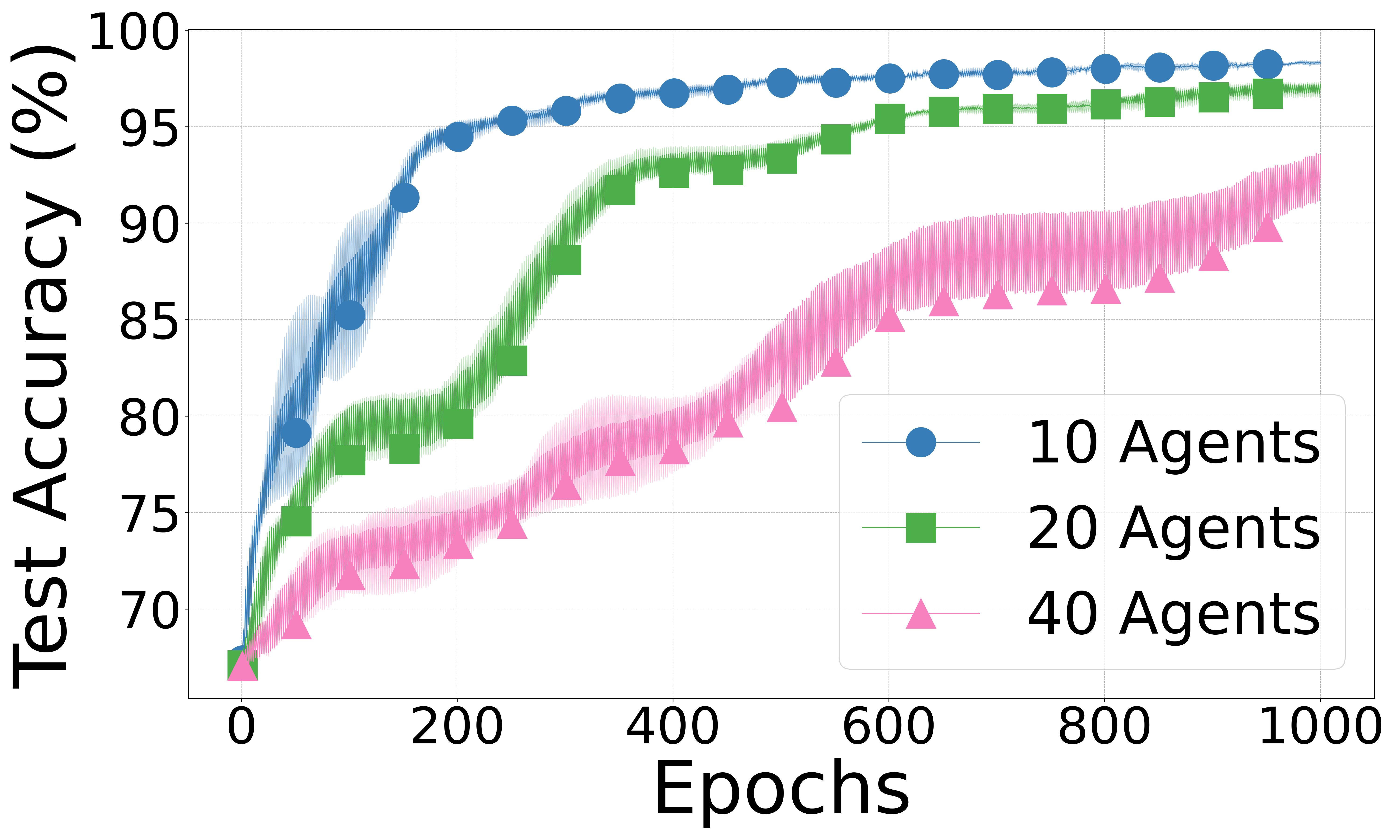}
        \caption{Scalability: FC, IID}
        \label{fig:scalability}
    \end{subfigure}

    \caption{Test accuracy on \textbf{Flowers} with \textbf{CLIP}: 
(\subref{fig:low_rank}) DeCAF with varying ranks (IID, FC); 
(\subref{fig:diff_acc_dl}) DeCAF vs. DLoRA accuracy gap; 
(\subref{fig:scalability}) DeCAF scalability with agent count (40 shots/class, rank 2).}
    \label{fig:comparison_study}
\end{figure*}
\section{Numerical Results}
\label{numerical_results}
In this section, we evaluate the performance of our proposed algorithms across various model architectures and datasets. We benchmark these algorithms against established baselines: Dec-LoRA (\textbf{DLoRA})~\cite{ghiasvand2025decentralized}, Federated Averaging (\textbf{FedAvg})~\cite{mcmahan2023communicationefficientlearningdeepnetworks} and Local Fine-tuning (\textbf{Local})~\cite{stich2018local}. While several variants of Federated LoRA exist (Table~\ref{tab:comparison}), we use the vanilla FedAvg as the primary baseline to evaluate the performance of the proposed decentralized algorithm against federated algorithms. Our evaluations consider different data distributions (IID and non-IID) and varying numbers of agents (10, 20, and 40).
For the VLM tasks, we adopt the data distribution framework from~\cite{saadati2024dimat}, which encompasses both IID and non-IID scenarios. Additional details regarding the construction of mixing matrices for different network topologies, extended LLM-based evaluation results, and a comprehensive analysis of computation and communication overhead are provided in the supplementary material (Section~\ref{additional_nu_results}).

\noindent\textbf{DeCAF for VLM.}
Table~\ref{tab:evaluation_vlm_all_updated} summarizes the test accuracy of various algorithms applied to the \textbf{Flowers}~\cite{Nilsback2008AutomatedFC}, \textbf{UCF}~\cite{soomro2012ucf101dataset101human}, and \textbf{Food101}~\cite{10.1007/978-3-319-10599-4_29} datasets under both IID and Non-IID distributions trained for 500 epochs, with a communication frequency of 1 epoch in federated and decentralized settings. The evaluation employs the \textbf{CLIP}~\cite{radford2021learningtransferablevisualmodels} model with a ViT-B/16 backbone. It has been fine-tuned using LoRA with a rank of 2, scaling factor $\eta=1$, and a dropout rate of 0.25, utilizing a batch size of 32 and a cosine learning rate scheduler. This follows the methodology outlined in the paper ~\cite{zanella2024low}. It is worth mentioning that the centralized results are directly sourced from ~\cite{zanella2024low} for 16 shots (\# of data samples per class). However, to maintain a true IID distribution, we used 40 shots for all other algorithms as the minimum number. Each setting receives a total of 40 labeled examples (shots) per class. For an IID distribution across 10 agents, this means each agent has 4 shots per class. In the Non-IID setting, each agent receives 40 shots from only 1 out of the 10 total classes.
In Table~\ref{tab:evaluation_vlm_all_updated}, DeCAF ranks as the best or the second-best performer, showcasing its effectiveness. The DLoRA-FA variant, though beneficial in reducing communication and computational overhead, comes with a trade-off in accuracy. This trend is similarly observed in FedAvg and Local baselines, 
except for the Food101 dataset, where even the centralized model shows limited performance gains, possibly due to the lack of regularization, when using  cross-entropy loss~\cite{zanella2024low}.

    

\noindent\textbf{DeCAF for LLM.} Table~\ref{tab:evaluation_metrics} presents the F1 Score for various algorithms across two datasets: \textbf{WIC}~\cite{pilehvar2019wicwordincontextdatasetevaluating}, and \textbf{Boolq}~\cite{clark2019boolq}, using a network of 10 agents. A pre-trained \textbf{LLAMA2-7B} model~\cite{touvron2023llama2openfoundation} was fine-tuned with LoRA, using a rank of 4, scaling factor $\eta = 8$, dropout rate of 0.1, and a batch size of 32. Training was performed for 150 batches, with communication occurring every 3 batches in both federated and decentralized settings.
Similar to the results observed for VLMs, DeCAF achieves either the highest or second-highest score in all LLM tasks. Additional results on instruction tuning with LLMs are provided in the SM~\ref{additional_nu_results}.

\noindent\textbf{Ablation Studies.}
\textbf{Different Topologies}: Fig.\ref{fig:topology_iid} and Fig.\ref{fig:topology_non_iid} present the impact of different network topologies on model performance using DeCAF under both IID and non-IID data distributions. In the IID setting, all topologies consistently yield strong performance. In contrast, under the more challenging non-IID setting, while bipartite and ring topologies introduce more noise and variability, they still achieve high accuracy. Notably, as shown in Fig.\ref{fig:algorithm}, DeCAF with a ring topology in the non-IID scenario—which represents an extremely difficult learning environment—still outperforms other decentralized learning algorithms. On the other hand, DLoRA with a ring topology fails to converge under non-IID conditions due to heightened data heterogeneity, which increases the error bound described in Theorem\ref{dlora-theo} and leads to divergence.
\textbf{Low Rank}: 
Fig.~\ref{fig:low_rank} presents test accuracy across ranks from 1 to 64, showing that higher ranks generally lead to better performance, consistent with Corollary~\ref{corollary_1}. We also observe that models converge to similar final accuracy levels, aligning with our theoretical insights.
\textbf{DeCAF vs. DLoRA}: As shown in Fig.~\ref{fig:diff_acc_dl}, the accuracy gap between DeCAF and DLoRA after 100 epochs of zero-shot training is more pronounced at lower ranks but narrows as the rank increases. This observation supports Theorem~\ref{theorem_2}, which predicts improved performance parity with higher ranks. Notably, this gap appears before full convergence—after convergence, even lower-rank models achieve accuracy close to higher-rank ones, consistent with our earlier observations.  
\textbf{Scalability}: As shown in Fig.~\ref{fig:scalability}, DeCAF's test accuracy remains largely consistent across different numbers of agents. Convergence occurs later in larger networks because each agent has fewer data samples when \( N \) is larger. These results suggest that DeCAF can be effectively deployed at various scales without significantly sacrificing performance.

\textbf{Limitations.}
While our results demonstrate strong performance, DeCAF warrants further evaluation across a wider range of tasks and application domains to ensure generalizability. Additionally, some observed phenomena—such as initial performance fluctuations or unexpected behaviors in ring topologies under non-IID data distributions—currently lack a theoretical explanation and present valuable directions for future study.

\section{Conclusions and Broader Impact}
\label{conclusions}
In this work, we introduced \textbf{DeCAF} and \textbf{DLoRA-FA}, two decentralized fine-tuning methods for vision-language and large language models that address key challenges in non-IID distributed learning. DeCAF mitigates consensus interference by applying truncated SVD (TSVD) during model merging, while DLoRA-FA improves efficiency by freezing the LoRA $\mathbf{A}$ matrix, reducing both computational and communication overhead. Both methods leverage low-rank adaptation for scalable, privacy-preserving learning and come with theoretical guarantees on convergence and low-rank error bounds. Our extensive experiments on large vision and language models confirm the effectiveness of these approaches across diverse datasets, agent configurations, and network topologies. Beyond academic benchmarks, DeCAF and DLoRA-FA hold strong potential for deployment in high-stakes real-world scenarios—such as healthcare, finance, and other privacy-sensitive domains—where decentralized, efficient, and secure model fine-tuning is essential.


\clearpage
\bibliographystyle{unsrt}
\bibliography{main}
\clearpage

\appendix
\section{Additional Analysis and Results}
\label{sec:additional_results}
This section presents the missing analysis and proof as well as the additional results. We start with the additional analysis to the main contents.
\subsection{Additional Related Work on Federated Learning with LoRA}\label{related_work_fl}
Parallel to decentralized learning, federated learning (FL) presents an alternative paradigm for distributed training with similar privacy and personalization goals but follows a different architectural approach. LoRA has recently been integrated into FL settings, where privacy concerns and decentralized data storage are addressed. One such work is Flora~\cite{hao2024floralowrankadapterssecretly}, which highlights LoRA's inherent ability to function as a gradient compressor, reducing the memory and communication costs of optimization. 
Another significant contribution in this area is the development of Federated LoRA with Sparse Communication (FLASC)~\cite{kuo2024federatedlorasparsecommunication}, which applies sparsity techniques to LoRA during federated learning to reduce communication costs. 
FedMS~\cite{wu2023fedms} introduces a two-stage federated training strategy combining global and local models with sparse adaptation (Sparsely Activated LoRA). Additionally, the Federated Instruction Tuning (FedIT) framework~\cite{zhang2024buildingfederatedgptfederated} demonstrates the effectiveness of federated approaches in instruction-tuning large language models (LLMs). This framework capitalizes on user-generated text data to train models while preserving privacy. OpenFedLLM~\cite{ye2024openfedllmtraininglargelanguage} further expands the federated learning landscape by offering a comprehensive framework for training LLMs on decentralized private data, showcasing superior performance compared to centralized approaches, particularly in privacy-sensitive domains like finance~\cite{liu2023differentially}.

Additionally, to address the issue of model averaging or consensus interference when optimizing low-rank matrices in a FL setting, recent works strategically developed methods freezing one low-rank matrix~\cite{sun2024improving} or only selectively sharing one matrix~\cite{guo2024selective}, leading to the higher communication and computational efficiency, but at the cost of model capabilities. Such a tradeoff has not yet theoretically been analyzed either. 
\subsection{Full Table for Comparison}
\begin{table*}[ht!]
 \caption{Comparison between different methods. $N$: \# of agents, $T$: the number of iterations, $r$: the low rank satisfying $r\ll\textnormal{min}\{d,k\}$, $d$: the input dimension, $k$: the output dimension, Comm.: cost per communication round, $p$: the partial participation percentage in $(0,1]$, $N_b$:  total number of non-zero elements in the mixing matrix $\mathbf{\Pi}$ (defined below), Smoo.Gua.: smoothness guarantee.}
  \label{tab:comparison_full}
  \centering
  \begin{tabular}{cccccc}
    \toprule
    Method & Rate & Low-rank & LLM/VLM & Comm. & Smoo.Gua.\\
    \midrule
    OpenFedLLM~\cite{ye2024openfedllm} & N/A & \xmark & \cmark/\xmark & $\mathcal{O}((pN+1)kd)$ & \xmark\\
    pFedLoRA~\cite{yi2023fedlora} & $\mathcal{O}(\frac{1}{\sqrt{T}})$ & \cmark & \xmark/\cmark & $\mathcal{O}((pN+1)(k+d)r)$ & \xmark\\
    FedIT~\cite{zhang2024towards}&N/A & \xmark & \cmark/\xmark & $\mathcal{O}((pN+1)kd)$ &\xmark\\
    FLoRA~\cite{wang2024flora} & N/A & \cmark & \cmark/\xmark & $\mathcal{O}((pN+1)(k+d)r)$ & \xmark\\
    FFA-LoRA~\cite{sun2024improving} & N/A & \cmark & \cmark/\xmark & $\mathcal{O}((pN+1)(k+d)r)$ & \cmark \\
    \hline
    Mix-of-Show~\cite{gu2024mix} & N/A & \cmark & \xmark/\cmark & $\mathcal{O}((pN+1)(k+d)r)$ & \xmark\\
    PCFT~\cite{wagner2024personalized}&N/A & \cmark & \cmark/\xmark & $\mathcal{O}(N_b(k+d)r)$ & \xmark\\
    Dec-LoRA~\cite{ghiasvand2025decentralized} & $\mathcal{O}(\frac{1}{T^{1/3}})$ & \cmark & \cmark/\xmark & $\mathcal{O}(N_b(k+d)r)$ & \xmark\\
    \hline
    \textbf{DLoRA} & $\mathcal{O}(\frac{1}{\sqrt{T}})$ & \cmark & \cmark/\cmark & $\mathcal{O}(N_b(k+d)r)$ &\cmark\\
    \textbf{DeCAF} & $\mathcal{O}(\frac{1}{\sqrt{T}})$ & \cmark & \cmark/\cmark & $\mathcal{O}(N_b(k+d)r)$ &\cmark\\
    \textbf{DLoRA-FA} & $\mathcal{O}(\frac{1}{\sqrt{T}})$ & \cmark & \cmark/\cmark & $\mathcal{O}(N_bdr)$ & \cmark\\
    \bottomrule
  \end{tabular}
\end{table*}
\subsection{Additional Analysis}\label{additional_analysis}
We first present a result regarding matrix product that serves to characterize the consensus analysis. To consider multiple agents, we will have the expanded mixing matrix of the Kronecker product between $\mathbf{\Pi}$ and $\mathbf{I}_n$, $\mathbf{P}=\mathbf{\Pi}\otimes \mathbf{I}_n$, but the magnitudes of eigenvalues of $\mathbf{P}$ remain the same through Theorem~\ref{kronecker_prod} presented in the following and the fact that eigenvalue of $\mathbf{I}_n$ is 1.
\begin{theorem}\cite{schacke2004kronecker}\label{kronecker_prod}
    Let $\mathbf{C}\in\mathbb{R}^{N\times N}$ and $\mathbf{D}\in\mathbb{R}^{n\times n}$, with eigenvalue $\lambda\in s(\mathbf{C})$ with corresponding eigenvector $x\in\mathbb{C}^{N}$, and $\mu\in s(\mathbf{D})$ with corresponding eigenvector $y\in\mathbb{C}^{n}$, where $s(\cdot)$ signifies the spectrum of a matrix. Then $\lambda\mu$ is an eigenvalue of $\mathbf{C}\otimes \mathbf{D}$ with corresponding eigenvector $x\otimes y\in\mathbb{C}^{nN}$. Any eigenvalue of $\mathbf{C}\otimes \mathbf{D}$ arises as such a product of eigenvalues of $\mathbf{C}$ and $\mathbf{D}$.
\end{theorem}
Theorem~\ref{kronecker_prod} suggests that $\mathbf{P}$ follows Assumption~\ref{assum_5} to have the property of $\textnormal{max}\{|\lambda_2(\mathbf{P})|, |\lambda_{nN}(\mathbf{P})|\}\leq \sqrt{\rho}<1$. 

\subsubsection{Model consensus interference}\label{model_consensus_inter}
We include technical detail for the model consensus interference, which appears in both federated learning and decentralized learning. 

Let us consider a decentralized learning task involving only two clients, with the same size of datasets. If we apply LoRA to clients locally, the global update induced by the FedAvg algorithm at the server would be:
\begin{equation}\label{eq_29}
    \mathbf{W}^+=\mathbf{W}_0+\frac{1}{2}(\mathbf{B}^1+\mathbf{B}^2)\times \frac{1}{2}(\mathbf{A}^1+\mathbf{A}^2),
\end{equation}
which is not exactly the same as what it should be as follows:
\begin{equation}\label{eq_30}
    \hat{\mathbf{W}}=\mathbf{W}_0+\frac{1}{2}(\mathbf{B}^2\mathbf{A}^2+\mathbf{B}^1\mathbf{A}^1).
\end{equation}
To address this issue, LoRA-FA can be leveraged in this context such that
\begin{equation}\label{eq_31}
    \mathbf{W}^+=\mathbf{W}_0+\frac{1}{2}(\mathbf{B}^1+\mathbf{B}^2)\times \mathbf{A}_0,
\end{equation}
which is the same as
\begin{equation}\label{eq_32}
    \hat{\mathbf{W}}=\mathbf{W}_0+\frac{1}{2}(\mathbf{B}^2\mathbf{A}_0+\mathbf{B}^1\mathbf{A}_0).
\end{equation}
$\mathbf{A}_0$ is the frozen low-rank matrix.
It can be observed that if we tune $\mathbf{A}$ matrix alone, the discordance from model consensus is also mitigated. However, \cite{zhu2024asymmetry} has provably shown that tuning $\mathbf{A}$ alone yields a worse generalization bound, compared to tuning $\mathbf{B}$ alone. Additionally, the latter one would require us to strategically select another initialization for $\mathbf{B}$ as it was 0 in LoRA.

\subsubsection{Smoothness guarantee issue}\label{smooth_issue}
Next, we present a formal result to show the smoothness guarantee issue in the vanilla objective $f^i$ when adopting LoRA, $\mathbf{W}(\mathbf{A}, \mathbf{B})=\mathbf{W}_0+\mathbf{B}\mathbf{A}$, without any additional conditions. We use $\mathbf{W}(\mathbf{A}, \mathbf{B})$ to distinguish if the parameter $\mathbf{W}$ has been used with LoRA or not.

\begin{theorem}
    Denote by $f^i(\mathbf{W}^i,\mathcal{D}_i)$ the local loss function given the weight and dataset for agent $i$. For a low-rank decomposition on model parameter such that $\mathbf{W}^i(\mathbf{A}^i, \mathbf{B}^i)=\mathbf{W}_0+\mathbf{B}^i\mathbf{A}^i$. Then, if both $\mathbf{A}^i$ and $\mathbf{B}^i$ are trainable and $f^i(\mathbf{W}^i,\mathcal{D}_i)$ is smooth with modulus $L$, the loss function $f^i(\mathbf{W}^i(\mathbf{A}^i, \mathbf{B}^i))$ has no smoothness guarantees.
\end{theorem}
\begin{proof}
    For the ease of notation, we introduce the variable $\bm{w}^i:=(\mathbf{A}^i,\mathbf{B}^i)$. To show the conclusion, we establish a counter-example such that the function is not smooth with respect to $\bm{w}^i$. Suppose that $\mathbf{W}^i, \mathbf{A}^i, \mathbf{B}^i\in\mathbb{R}^{d\times d}$ and that $f^i(\mathbf{W}^i)=\frac{1}{2}\|\mathbf{W}^i\|_F^2$ with $\mathbf{W}_0=0$. Therefore, we consider a sequence $\{\bm{w}^i_t\}_{t\in\mathbb{N}}$ such that $\bm{w}^i_t=(\mathbf{A}^i_t,\mathbf{B}^i_t)=(t\mathbf{I}_d,t\mathbf{I}_d)$. Thus, the following relationship holds
    \begin{equation}
    \begin{split}
        &\text{lim}_{t\to\infty}\frac{\|\nabla_{\bm{w}^i}\mathbf{W}^i(\mathbf{A}^i_t,\mathbf{B}^i_t)-\nabla_{\bm{w}^i}\mathbf{W}^i(\mathbf{A}^i_0,\mathbf{B}^i_0)\|_F}{\|\bm{w}^i_t-\bm{w}^i_0\|_F}\\&
        =\text{lim}_{t\to\infty}\frac{\|\nabla_{\mathbf{A}^i}\mathbf{W}^i(\mathbf{A}^i_t,\mathbf{B}^i_t)-\nabla_{\mathbf{A}^i}\mathbf{W}^i(\mathbf{A}^i_0,\mathbf{B}^i_0)\|_F}{\|\mathbf{A}^i_t-\mathbf{A}^i_0\|_F+\|\mathbf{B}^i_t-\mathbf{B}^i_0\|_F}\\&
        +\text{lim}_{t\to\infty}\frac{\|\nabla_{\mathbf{B}^i}\mathbf{W}^i(\mathbf{A}^i_t,\mathbf{B}^i_t)-\nabla_{\mathbf{B}^i}\mathbf{W}^i(\mathbf{A}^i_0,\mathbf{B}^i_0)\|_F}{\|\mathbf{A}^i_t-\mathbf{A}^i_0\|_F+\|\mathbf{B}^i_t-\mathbf{B}^i_0\|_F}\\&
        +\text{lim}_{t\to\infty}\frac{\|t^3\mathbf{I}_d\|_F+\|t^3\mathbf{I}_d\|_F}{\|t\mathbf{I}_d\|_F+\|t\mathbf{I}_d\|_F} =\infty
    \end{split}
    \end{equation}
    Therefore, based on the existing counter-example, we can see that although $f^i(\mathbf{W}^i, \mathcal{D}_i)$ is smooth with modulus 1, the function is not smooth with respect to $\bm{w}^i$. Equivalently, the original loss function is not smooth with the low-rank adapter.
\end{proof}
In the main content, we have proposed one therapy to fix the above issue such that the loss function $f^i$ still maintains the smoothness property, but with a new smoothness constant. The result is summarized in Lemma~\ref{lemma_1}.

\subsubsection{Proof for Lemma~\ref{lemma_1}}\label{proof_lemma_1}
\begin{proof}
        With $\bm{w}=(\mathbf{A}, \mathbf{B})$, and 
        recalling the gradient of $f^i$ on $\mathbf{W}$, $\nabla_\mathbf{W} f^i$, we can obtain the gradient on $\bm{w}$ being $[\nabla_{\mathbf{A}}f^i(\mathbf{W}),\nabla_{\mathbf{B}}f^i(\mathbf{W})]$. Based on the chain rule, we have $\nabla_{\mathbf{A}}f^i(\mathbf{W})=\frac{\eta}{r}\mathbf{B}^\top\nabla_{\mathbf{W}}f^i$ and $\nabla_{\mathbf{B}}f^i(\mathbf{W})=\frac{\eta}{r}\nabla_{\mathbf{W}}f^i\mathbf{A}^\top$. 
    Given this and the fact that $\mathbf{W} = \mathbf{W}_0 + \frac{\eta}{r}\mathbf{B}\mathbf{A}$ and $\mathbf{W}' = \mathbf{W}_0 + \frac{\eta}{r}\mathbf{B}'\mathbf{A}'$, we have the following relationship,
    \begin{equation}
        \begin{split}
            &\frac{\|\nabla_{\bm{w}} f^i(\bm{w},\mathbf{W}_0)-\nabla_{\bm{w}'} f^i(\bm{w}',\mathbf{W}_0)\|_F}{\|\bm{w}-\bm{w}'\|_F}\\&=\frac{\eta}{r}\frac{\|\mathbf{B}^\top\nabla_{\mathbf{W}}f^i-\mathbf{B}'^\top\nabla_{\mathbf{W}'}f^i\|_F+\|\nabla_{\mathbf{W}}f^i\mathbf{A}^\top-\nabla_{\mathbf{W}'}f^i\mathbf{A}'^\top\|_F}{\|\bm{w}-\bm{w}'\|_F}\\&=\frac{\eta}{r} \bigg(\frac{\|\mathbf{B}^\top\nabla_{\mathbf{W}}f^i-\mathbf{B}^\top\nabla_{\mathbf{W}'}f^i+\mathbf{B}^\top\nabla_{\mathbf{W}'}f^i-\mathbf{B}'^\top\nabla_{\mathbf{W}'}f^i\|_F}{\|\bm{w}-\bm{w}'\|_F}\\&+\frac{\|\nabla_{\mathbf{W}}f^i\mathbf{A}^\top-\nabla_{\mathbf{W}'}f^i\mathbf{A}^\top+\nabla_{\mathbf{W}'}f^i\mathbf{A}^\top-\nabla_{\mathbf{W}'}f^i\mathbf{A}'^\top\|_F}{\|\bm{w}-\bm{w}'\|_F}\bigg)
        \end{split}
    \end{equation}
As $\mathbf{B}^\top\nabla_{\mathbf{W}}f^i-\mathbf{B}^\top\nabla_{\mathbf{W}'}f^i+\mathbf{B}^\top\nabla_{\mathbf{W}'}f^i-\mathbf{B}'^\top\nabla_{\mathbf{W}'}f^i=\mathbf{B}^\top(\nabla_{\mathbf{W}}f^i-\nabla_{\mathbf{W}'}f^i)+(\mathbf{B}^\top-\mathbf{B}'^\top)\nabla_{\mathbf{W}'}f^i$
According to Triangle inequality and Cauchy-Schwartz inequality, we can obtain
\begin{equation}
    \begin{split}
        &\|\mathbf{B}^\top\nabla_{\mathbf{W}}f^i-\mathbf{B}^\top\nabla_{\mathbf{W}'}f^i+\mathbf{B}^\top\nabla_{\mathbf{W}'}f^i-\mathbf{B}'^\top\nabla_{\mathbf{W}'}f^i\|_F\\&\leq \|\mathbf{B}^\top(\nabla_{\mathbf{W}}f^i-\nabla_{\mathbf{W}'}f^i)\|_F+\|(\mathbf{B}^\top-\mathbf{B}'^\top)\nabla_{\mathbf{W}'}f^i\|_F\\&\leq\|\mathbf{B}^\top\|_F\|\nabla_{\mathbf{W}}f^i-\nabla_{\mathbf{W}'}f^i\|_F+\|\mathbf{B}^\top-\mathbf{B}'^\top\|_F\|\nabla_{\mathbf{W}'}f^i\|_F.
    \end{split}
\end{equation}
Similarly, we have $\|\nabla_{\mathbf{W}}f^i\mathbf{A}^\top-\nabla_{\mathbf{W}'}f^i\mathbf{A}^\top+\nabla_{\mathbf{W}'}f^i\mathbf{A}^\top-\nabla_{\mathbf{W}'}f^i\mathbf{A}'^\top\|_F\leq\|\nabla_{\mathbf{W}}f^i-\nabla_{\mathbf{W}'}f^i\|_F\|\mathbf{A}^\top\|_F+\|\nabla_{\mathbf{W}'}f^i\|_F\|\mathbf{A}^\top-\mathbf{A}'^\top\|_F$. With these relationships in hand, we have
\begin{equation}
    \begin{split}
        &\frac{\|\nabla_{\bm{w}} f^i(\bm{w},\mathbf{W}_0)-\nabla_{\bm{w}'} f^i(\bm{w}',\mathbf{W}_0)\|_F}{\|\bm{w}-\bm{w}'\|_F}\\&\leq\frac{\eta}{r}\bigg(\frac{\|\nabla_{\mathbf{W}}f^i-\nabla_{\mathbf{W}'}f^i\|_F(\|\mathbf{A}^\top\|_F+\|\mathbf{B}^\top\|_F)}{\|\mathbf{A}-\mathbf{A}'\|_F+\|\mathbf{B}-\mathbf{B}'\|_F}\\&+
        \frac{\|\nabla_{\mathbf{W}'}f^i\|_F(\|\mathbf{A}^\top-\mathbf{A}'^\top\|_F+\|\mathbf{B}^\top-\mathbf{B}'^\top\|_F)}{\|\mathbf{A}-\mathbf{A}'\|_F+\|\mathbf{B}-\mathbf{B}'\|_F}\bigg)
        \\&
        \leq\frac{\eta}{r}\bigg(\frac{L\|\mathbf{W}-\mathbf{W}'\|_F(\|\mathbf{A}^\top\|_F+\|\mathbf{B}^\top\|_F)}{\|\mathbf{A}-\mathbf{A}'\|_F+\|\mathbf{B}-\mathbf{B}'\|_F}+G\bigg)\\&\leq\frac{ 2LC\sqrt{r}\sigma+G}{r},
    \end{split}
\end{equation}
where the second inequality follows from Assumption~\ref{assum_1} and the third inequality follows from Assumption~\ref{assum_2} and Assumption~\ref{assum_3}.
\end{proof}
\begin{remark}
    From Lemma~\ref{lemma_1}, it can be observed that all the losses $f^i$ by using LoRA is \textit{smooth} with a new constant $2LC\sqrt{r}\sigma+G$. Compared to the original smoothness constant $L$, with the low-rank decomposition and tuning both $\mathbf{A}^i$ and $\mathbf{B}^i$ yields a larger smoothness constant if the rank $r$ is upper bound by $4C^2c^2\eta^2$. This intuitively implies that the loss $f^i$ with low-rank decomposition becomes less smooth, but instead the gradient may change faster. This can assist in convergence for the early phase of the optimization, particularly if the gradient is large at a point. Since a larger smoothness constant implies more dramatic changes in gradients, which benefits the gradient decaying. However, less smooth objective function may slow down the convergence as well, especially in the later phase, thus negatively affecting the convergence error. 
\end{remark}
We have now obtained a new smoothness constant for the objective function $f^i$ with respect to LoRA matrices $(\mathbf{A},\mathbf{B})$, which assists in the proof in the main result.

\subsubsection{Proof for Theorem~\ref{dlora-theo}}
With abuse of notation, we use some upper bold characters to represent vectors after they are expanded. 
Define
\[\mathbf{V}_t=[\bm{v}^1_t;\bm{v}^2_t;...;\bm{v}^N_t]^\top\in\mathbb{R}^{nN},\]\[\mathbf{G}_t=[\bm{g}^1_t;\bm{g}^2_t;...;\bm{g}^N_t]^\top\in\mathbb{R}^{nN},\]\[\mathbf{H}_t=[\nabla f^1(\bm{v}^1_t,\bm{\theta}_0);\nabla f^2(\bm{v}^2_t,\bm{\theta}_0);...;\nabla f^N(\bm{v}^N_t,\bm{\theta}_0)]^\top\in\mathbb{R}^{nN},\]\[\mathbf{Q}=\frac{1}{mN}\mathbf{1}\mathbf{1}^\top_{nN}\in\mathbb{R}^{nN\times nN}\]
Without loss of generality, suppose that the initialization of $\mathbf{V}$ is $\mathbf{V}_0$ (which is not 0 due to the non-zero initialization of $\mathbf{A}$) throughout the rest of analysis. For \textsc{DLoRA}, we have
\begin{equation}
    \mathbf{V}_t = \mathbf{P}^{t}\mathbf{V}_0-\alpha\sum_{\tau=1}^{t-1}\mathbf{P}^{t-1-\tau}\mathbf{G}_\tau
\end{equation}
Left multiplying the above equation by $\mathbf{I-Q}$ yields the following relationship
\begin{equation}\label{eq_14}
    (\mathbf{I-Q})\mathbf{V}_t = (\mathbf{I}-\mathbf{Q})\mathbf{P}^{t}\mathbf{V}_0-\alpha\sum_{\tau=1}^{t-1}(\mathbf{I-Q})\mathbf{P}^{t-1-\tau}\mathbf{G}_\tau,
\end{equation}
which will serve to characterize the optimal error bound. By taking the squared Frobenius norm and expectation on both sides, we have
\begin{equation}\label{consensus}
    \mathbb{E}[\|(\mathbf{I-Q})\mathbf{V}_t\|_F^2]\leq 2\mathbb{E}[\|(\mathbf{I}-\mathbf{Q})\mathbf{P}^{t}\mathbf{V}_0\|_F^2]+2\alpha^2\mathbb{E}[\|\sum_{\tau=1}^{t-1}(\mathbf{I-Q})\mathbf{P}^{t-1-\tau}\mathbf{G}_\tau\|_F^2].
\end{equation}
The left side of above equation is equivalent to $\mathbb{E}[\frac{1}{N}\sum_{i=1}^N\|\bm{v}^i_t-\bm{\bar{v}}_t\|^2]$. To further analyze the Eq.~\ref{consensus}, we investigate the second term of its right side in the following.
\begin{equation}\label{eq_16}
\begin{split}
    &\alpha^2\mathbb{E}[\|\sum_{\tau=1}^{t-1}(\mathbf{I-Q})\mathbf{P}^{t-1-\tau}\mathbf{G}_\tau\|_F^2]\leq \\&2\alpha^2\underbrace{\mathbb{E}[\|\sum_{\tau=1}^{t-1}(\mathbf{I-Q})\mathbf{P}^{t-1-\tau}(\mathbf{G}_\tau-\mathbf{H}_\tau)\|_F^2]}_{T_1}\\&
    +2\alpha^2\underbrace{\mathbb{E}[\|\sum_{\tau=1}^{t-1}(\mathbf{I-Q})\mathbf{P}^{t-1-\tau}\mathbf{H}_\tau\|_F^2]}_{T_2},
\end{split}
\end{equation}
which follows by using the basic inequality $\|\mathbf{a}+\mathbf{b}\|^2\leq 2\|\mathbf{a}\|^2+2\|\mathbf{b}\|^2$. We will next study the upper bounds for $T_1$ and $T_2$, respectively. Before that, we state two key lemmas to manipulate $\mathbf{G}_t - \mathbf{H}_t$.

\begin{lemma}\label{lemma_2}
    Let Assumption~\ref{assum_4} hold. Suppose that $\mathbb{E}[\bm{g}^i] = \nabla f^i(\bm{v}^i,\bm{\theta}_0), \forall i\in\mathcal{V}$. Then, we have the following relationship
    \begin{equation}
        \mathbb{E}[\|\frac{1}{N}\sum_{i=1}^N\bm{g}^i\|^2]\leq\frac{1}{N}\zeta^2+\mathbb{E}[\|\frac{1}{N}\sum_{i=1}^N\nabla f^i(\bm{v}^i,\bm{\theta}_0)\|^2].
    \end{equation}
\end{lemma}
The proof for the Lemma~\ref{lemma_1} follows similarly Lemma 1 in~\cite{yu2019linear} and we skip it in this context.
We are now able to bound $T_1$ and $T_2$. By following the similar proof techniques and adapting the analysis in~\cite{yu2019linear}, the following bounds are obtained accordingly.
\begin{equation}\label{eq_21}
    T_1\leq \frac{N\zeta^2}{1-\rho}
\end{equation}
\begin{equation}\label{eq_22}
\begin{split}
    &T_2\leq \frac{1}{1-\sqrt{\rho}}[8\hat{L}^2\sum^t_{\tau=1}\rho^{(t-\tau)/2}\mathbb{E}[\|\mathbf{(I-Q)\mathbf{V}_\tau}\|^2]\\&
    +4N\sum_{\tau=1}^t\rho^{(t-\tau)/2}\mathbb{E}[\|\frac{1}{N}\sum_{i=1}^N\nabla f^i(\bm{v}^i_\tau,\bm{\theta}_0)\|^2]]\\&+\frac{4N\kappa^2}{1-\sqrt{\rho}}
\end{split}
\end{equation}
Based on the upper bounds for $T_1$ and $T_2$, we can obtain the upper bound for $\frac{1}{N}\sum_{i=1}^N\mathbb{E}[\|\bm{v}^i_k-\bm{\bar{v}}_t\|^2]$.
\begin{lemma}\label{lemma_3}
    Let Assumptions~\ref{assum_4} and~\ref{assum_5} hold. For $\bm{v}_t^i$ defined by Algorithm~\ref{alg:dlora},
    if step size $\alpha\leq \frac{1-\sqrt{\rho}}{4\sqrt{2}\hat{L}}$, then $\forall T\geq 1$, the following relationship holds true
    \begin{equation}
    \begin{split}
        &\sum_{t=1}^T\frac{1}{N}\sum_{i=1}^N\mathbb{E}[\|\bm{v}^i_t-\bm{\bar{v}}_t\|^2]\leq \sum_{t=1}^T\rho^{t}\|(\mathbf{I-Q})\mathbf{V}_0\|_F^2+\frac{4T\alpha^2\zeta^2}{1-\rho}\\&+\frac{16\alpha^2}{(1-\sqrt{\rho})^2}\sum^T_{t=1}\mathbb{E}[\|\frac{1}{N}\sum_{i=1}^N\nabla f^i(\bm{v}^i_t,\bm{\theta}_0)\|^2]\\&+\frac{16T\alpha^2\kappa^2}{(1-\sqrt{\rho})^2},
    \end{split}
    \end{equation}
where $\bar{\bm{v}}_t=\frac{1}{N}\sum_{i=1}^N\bm{v}^i_t$.
\end{lemma}
Combining Eqs~\ref{eq_16},~\ref{eq_21}, and~\ref{eq_22} and summing $t$ over $\{1,2,...,T\}$, and following the proof techniques from~\cite{yu2019linear} can complete the proof for Lemma~\ref{lemma_3}.
With Lemma~\ref{lemma_3} in hand, we are ready to give the proof of Theorem~\ref{dlora-theo} in the following.
\begin{proof}
    We start with the descent inequality given by the smoothness of $f$ such that
    \begin{equation}\label{eq_24}
    \begin{split}
        &\mathbb{E}[f(\bar{\bm{v}}_{t+1},\bm{\theta}_0)]\leq \mathbb{E}[f(\bar{\bm{v}}_t,\bm{\theta}_0)] + \mathbb{E}[\langle\nabla f(\bar{\bm{v}}_t,\bm{\theta}_0),\bar{\bm{v}}_{t+1}-\bar{\bm{v}}_t\rangle]\\&+\frac{\hat{L}}{2}\mathbb{E}[\|\bar{\bm{v}}_{t+1}-\bar{\bm{v}}_t\|^2]
    \end{split}
    \end{equation}
This is due to the fact that for all $i$, $f^i$ is $\hat{L}$-smooth.
We first process the second term on the right side of above inequality. Replacing $\bar{\bm{v}}_{t+1}-\bar{\bm{v}}_t$ with $-\alpha\frac{1}{N}\sum_{i=1}^N\bm{g}^i(\bm{v}^i_t,\bm{\theta}_0)$ allows us to study $-\alpha\mathbb{E}[\langle\nabla f(\bar{\bm{v}}_k,\bm{\theta}_0),\frac{1}{N}\sum_{i=1}^N\bm{g}^i(\bm{v}^i_t,\bm{\theta}_0)\rangle]$.
Thus, we have
\begin{equation}
\begin{split}
&\langle\nabla f(\bar{\bm{v}}_t,\bm{\theta}_0),\frac{1}{N}\sum_{i=1}^N\bm{g}^i(\bm{v}^i_t,\bm{\theta}_0)\rangle=\frac{1}{2}(\|\nabla f(\bar{\bm{v}}_t,\bm{\theta}_0)\|^2+\\&\|\frac{1}{N}\sum_{i=1}^N\nabla f^i(\bm{v}^i_t,\bm{\theta}_0)\|^2-\|\nabla f(\bar{\bm{v}}_t,\bm{\theta}_0)-\frac{1}{N}\sum_{i=1}^N\nabla f^i(\bm{v}_t^i,\bm{\theta}_0)\|^2)\\&\geq \frac{1}{2}(\|\nabla f(\bar{\bm{v}}_t,\bm{\theta}_0)\|^2+\|\frac{1}{N}\sum_{i=1}^N\nabla f^i(\bm{v}^i_t,\bm{\theta}_0)\|^2\\&-\hat{L}^2\frac{1}{N}\sum_{i=1}^N\|\bar{\bm{v}}_t-\bm{v}^i_t\|^2).
\end{split}
\end{equation}
The last inequality follows from the smoothness assumption. Therefore, we have
\begin{equation}
\begin{split}
&\mathbb{E}[\langle\nabla f(\bar{\bm{v}}_t,\bm{\theta}_0),\bar{\bm{v}}_{t+1}-\bar{\bm{v}}_t\rangle]\leq-\frac{\alpha}{2}\mathbb{E}[\|\nabla f(\bar{\bm{v}}_t,\bm{\theta}_0)\|^2\\&+\|\frac{1}{N}\sum_{i=1}^N\nabla f^i(\bm{v}^i_t,\bm{\theta}_0)\|^2]+\frac{\alpha \hat{L}^2}{2}\frac{1}{N}\sum_{i=1}^N\mathbb{E}[\|\bar{\bm{v}}_t-\bm{v}^i_t\|^2].
\end{split}
\end{equation}
With Eq.~\ref{eq_24}, the following relationship can be obtained.
\begin{equation}
\begin{split}
&\mathbb{E}[f(\bar{\bm{v}}_{t+1},\bm{\theta}_0)]\leq \mathbb{E}[f(\bar{\bm{v}}_{t},\bm{\theta}_0)]-\frac{\alpha}{2}\mathbb{E}[\|\nabla f(\bar{\bm{v}}_t,\bm{\theta}_0)\|^2\\&+\|\frac{1}{N}\sum_{i=1}^N\nabla f^i(\bm{v}^i_t,\bm{\theta}_0)\|^2]+\frac{\alpha \hat{L}^2}{2}\frac{1}{N}\sum_{i=1}^N\mathbb{E}[\|\bar{\bm{v}}_t-\bm{v}^i_t\|^2]\\&+\frac{\hat{L}}{2}\mathbb{E}[\|\frac{1}{N}\sum_{i=1}^N\bm{g}^i_t\|^2].
\end{split}
\end{equation}
The last inequality holds due to 
$\bar{\bm{v}}_{t+1}-\bar{\bm{v}}_t=\frac{1}{N}\sum_{i=1}^N\bm{g}^i_t$.
With Lemma~\ref{lemma_1} and some mathematical manipulations, we have 
\begin{equation}
    \begin{split}
&\mathbb{E}[f(\bar{\bm{v}}_{t+1},\bm{\theta}_0)]\leq\mathbb{E}[f(\bar{\bm{v}}_t,\bm{\theta}_0)]-\frac{\alpha}{2}\mathbb{E}[\|\nabla f(\bar{\bm{v}}_t,\bm{\theta}_0)\|^2\\&+\|\frac{1}{N}\sum_{i=1}^N\nabla f^i(\bm{v}^i_t,\bm{\theta}_0)\|^2]+\frac{\alpha \hat{L}^2}{2}\frac{1}{N}\sum_{i=1}^N\mathbb{E}[\|\bar{\bm{v}}_t-\bm{v}^i_t\|^2]\\&+\frac{\hat{L}\alpha^2}{2}(\frac{\zeta^2}{N}+\mathbb{E}[\frac{1}{N}\sum_{i=1}^N\|\nabla f^i(\bm{v}^i_t,\bm{\theta}_0)\|^2])\\&=\mathbb{E}[f(\bar{\bm{v}}_t,\bm{\theta}_0)]-\frac{\alpha}{2}\mathbb{E}[\|\nabla f(\bar{\bm{v}}_t,\bm{\theta}_0)\|^2]-\\&(\frac{\alpha}{2}-\frac{\hat{L}\alpha^2}{2})\mathbb{E}[\|\frac{1}{N}\sum_{i=1}^N\nabla f^i(\bm{v}^i_t,\bm{\theta}_0)\|^2]\\& + \frac{\alpha \hat{L}^2}{2N}\sum_{i=1}^N\mathbb{E}[\|\bar{\bm{v}}_t-\bm{v}^i_t\|^2]+\frac{\hat{L}\alpha^2\zeta^2}{2N},
    \end{split}
\end{equation}
which implies the following inequality
\begin{equation}\label{eq_29}
\begin{split}
&\mathbb{E}[\|\nabla f(\bar{\bm{v}}_t,\bm{\theta}_0)\|^2]\leq\frac{2}{\alpha}(\mathbb{E}[f(\bar{\bm{v}}_t,\bm{\theta}_0)]-\mathbb{E}[f(\bar{\bm{v}}_{t+1},\bm{\theta}_0)])\\&-(1-\hat{L}\alpha)\mathbb{E}[\|\frac{1}{N}\sum_{i=1}^N\nabla f^i(\bm{v}^i_t,\bm{\theta}_0)\|^2]+\frac{\hat{L}^2}{N}\mathbb{E}[\|\bar{\bm{v}}_t-\bm{v}^i_t\|^2]\\&+\frac{\hat{L}\alpha\zeta^2}{N}.
\end{split}
\end{equation}
The above relationship is obtained by dividing $\alpha/2$ on both sides.
Summing $t$ over $\{1,2,...,T\}$ yields
\begin{equation}
\begin{split}
    &\sum_{t=1}^T\mathbb{E}[\|\nabla f(\bar{\bm{v}}_t,\bm{\theta}_0)\|^2]\leq \frac{2}{\alpha}(f(\bar{\bm{v}}_0,\bm{\theta}_0)-\mathbb{E}[f(\bar{\bm{v}}_T,\bm{\theta}_0)])\\&-(1-\hat{L}\alpha)\sum_{t=1}^T\mathbb{E}[\|\frac{1}{N}\sum_{i=1}^N\nabla f^i(\bm{v}^i_t,\bm{\theta}_0)\|^2]\\&+\hat{L}^2(\sum_{t=1}^T\rho^{t}\|(\mathbf{I-Q})\mathbf{V}_0\|_F^2+\frac{4T\alpha^2\sigma^2}{1-\rho}+\frac{16\alpha^2}{(1-\sqrt{\rho})^2}\sum_{t=1}^T\mathbb{E}[\|\frac{1}{N}\sum_{i=1}^N\nabla f^i(\bm{v}^i_t,\bm{\theta}_0)\|^2]\\&+\frac{16T\alpha^2\kappa^2}{(1-\sqrt{\rho})^2})+\frac{\hat{L}\alpha\zeta^2T}{N}\\&\leq \frac{2}{\alpha}(f(\bar{\bm{v}}_0,\bm{\theta}_0)-f^*)-(1-\hat{L}\alpha-\frac{16\hat{L}^2\alpha^2}{(1-\sqrt{\rho})^2})\sum_{t=1}^T\mathbb{E}[\|\frac{1}{N}\sum_{i=1}^N\nabla f^i(\bm{v}^i_t,\bm{\theta}_0)\|^2]\\&
    +\frac{4T\hat{L}^2\alpha^2\zeta^2}{1-\rho}+\frac{16T\hat{L}^2\alpha^2\kappa^2}{(1-\sqrt{\rho})^2}+\frac{\hat{L}T\alpha\zeta^2}{N} +\frac{\hat{L}^2\sum_{i=1}^N\|\bm{v}_0^i\|^2}{TN(1-\rho)}
\end{split}
\end{equation}
The last inequality is attained by substituting the conclusion from Lemma~\ref{lemma_3} into Eq.~\ref{eq_29}. 
Due to the condition for the step size $\alpha$, we know that $1-\hat{L}\alpha-\frac{16\hat{L}^2\alpha^2}{(1-\sqrt{\rho})^2}\geq0$, which would simplify the right side in the last inequality. As $\|(\mathbf{I-Q})\mathbf{V}_0\|_F^2= \frac{1}{N}\sum_{i=1}^N\|\bm{v}_0^i\|^2$,  hence,
\begin{equation}
\begin{split}
&\sum_{t=1}^T\mathbb{E}[\|\nabla f(\bar{\bm{v}}_t,\bm{\theta}_0)\|^2]\leq \frac{2}{\alpha}(f(\bar{\bm{v}}_0,\bm{\theta}_0)-f^*)+\frac{4T\hat{L}^2\alpha^2\zeta^2}{1-\rho}\\&+\frac{16T\hat{L}^2\alpha^2\kappa^2}{(1-\sqrt{\rho})^2}+\frac{T\hat{L}\alpha\zeta^2}{N}+\frac{\hat{L}^2\sum_{i=1}^N\|\bm{v}_0^i\|^2}{N(1-\rho)}
\end{split}
\end{equation}
The desirable result is obtained by dividing $T$ on both sides.
\end{proof}

\subsubsection{Proof for Corollary~\ref{corollary_1}}
\begin{proof}
    Substituting $\alpha\lesssim\sqrt{\frac{N}{T}
    }$ and $\hat{L} = \frac{\eta(2LC\sqrt{r}c+G)}{r}$ into the conclusion of Theorem~\ref{dlora-theo} yields the conclusion.
\end{proof}
\noindent\textbf{Decentralized vs. Local.} For DLoRA, once $\bm{v}^i_*$ is obtained, each agent has the \textit{personalized} optimal model parameters as $\mathbf{W}_*^i=\mathbf{W}_0+\mathbf{B}^i_*\mathbf{A}^i_*$. $\mathbf{B}^i_*\mathbf{A}^i_*$ has been optimized by communication and computation, and the neighborhood communication equips each agent with extra knowledge and abilities learned from other agents, enabling the collaborative learning to reduce the negative impact of data distribution diversity. Instead, one can perform local LoRA for all $N$ agents, and in this scenario there is no communication to ensure in-time consensus among different agents, inevitably resulting in a poor solution. Thereby, DLoRA strategically outperforms the local fine-tuning strategy, which will be validated by the empirical evidence in the next section.

\subsubsection{Analysis for DLoRA-FA}\label{analysis_dlora_fa}
In this subsection, we present the detailed analysis for the DLoRA with freezing $\mathbf{A}$ matrix, as shown in Figure~\ref{fig:dlora_fa_diagram}. We start with an algorithm framework for DLoRA-FA as follows. Note that now $\bm{w}$ is only $\mathbf{B}$.

\begin{figure*}
    \centering
    \includegraphics[width=\linewidth]{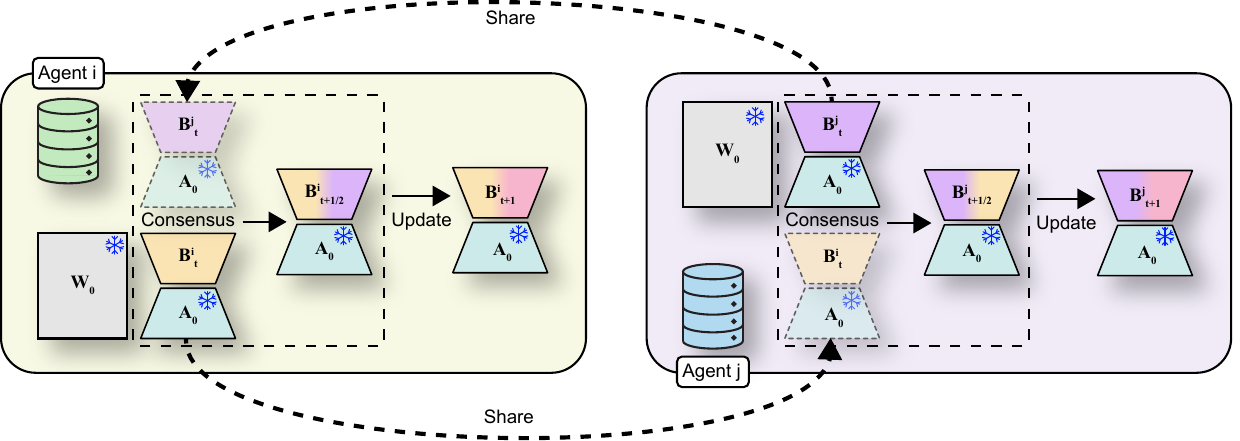}
    \caption{Illustration of DLoRA-FA by using two agents: agents $i$ and $j$ share their low-rank matrices (only $\mathbf{B}$) separately during the communication step, $t$, and conduct consensus on the received matrices from neighbors with its own. Once the consensus is done, each agent use the locally owned data to update the low-rank matrix $\mathbf{B}$. During this process, the pre-trained weights $\mathbf{W}_0$ and one low-rank matrix $\mathbf{A}_0$ are frozen. In DLoRA-FA, there is no matrix factorization as in DeCAF, which has been shown in Theorem~\ref{model_interference_theo}.}
    \label{fig:dlora_fa_diagram}
\end{figure*}

\begin{algorithm}
  \caption{\textsc{DLoRA-FA}}
  \label{alg:dlora_fa}
  \SetKwInOut{Input}{Input}
  \SetKwInOut{Output}{Output}
  \Input{mixing matrix $\mathbf{\Pi}=[\pi_{ij}]_{i,j\in N}, \pi_{ij}\in[0,1]$, the \# of iterations $T$, initialization of low rank adapters $\mathbf{B}_1^i, \forall i\in\mathcal{V}$, step size $\alpha$, frozen parameters $\mathbf{W}_0, \mathbf{A}_0$, $\mathcal{S}_i, i\in\mathcal{V}$, communication frequency $\tau$}
  \Output{$\{\mathbf{B}^i_T\}_{i=1}^N$}  
  \BlankLine
  \For{ $t$ in $1:T$ }
  { 
    \For{each agent $i\in\mathcal{V}$}
    { 
    Calculate the stochastic gradients $\bm{g}^{i,\mathbf{\mathbf{B}}}_t=\frac{1}{|\mathcal{S}_i|}\sum_{s\in\mathcal{S}_i}\nabla_\mathbf{\mathbf{B}} f^i_s(\bm{w}^i_t,\mathbf{A}_0,\mathbf{W}_0)$\;
    \eIf{$t$ mod $\tau$=0}
    {     Broadcast the low-rank matrices $\mathbf{B}^i_t$ to and receive $\mathbf{B}^j_t$ from all nodes in $Nb(i)$\;
    $\mathbf{B}^i_{t+1/2}=\sum_{j\in Nb(i)}\pi_{ij}\mathbf{B}^j_{t}$\;
    }{$\mathbf{B}^i_{t+1/2}=\mathbf{B}^i_{t}$\;}
    $\mathbf{B}^i_{t+1}=\mathbf{B}^i_{t+1/2}-\alpha \bm{g}^{i,\mathbf{B}}_t$\;    }
  }
\end{algorithm}
To characterize the analysis, we derive a new smoothness constant due to freezing $\mathbf{A}$ matrix. The following lemma states the upper bound of the Frobenius norm of $\mathbf{A}_0$ matrix.
\begin{lemma}\label{lemma_4}
    Let $\mathbf{A}_0$ be a Gaussian random matrix, where each entry is sampled independently from a normal distribution $\mathcal{N}(0,1)$. Then $\|\mathbf{A}_0\|_F\leq r+\sqrt{kr}$. 
\end{lemma}
\begin{proof}
    Since $\|\mathbf{A}_0\|_F=\sqrt{\sum_{l}\sigma_l(\mathbf{A}_0)^2}\leq\sqrt{r}\sigma_1$ as $r\ll\textnormal{min}\{d,k\}$, where $\sigma_l(\mathbf{A}_0)$ indicates the $i$-th singular value of $\mathbf{A}_0$. Based on Gordon's Theorem for Gaussian matrices~\cite{davidson2001local}, we have $\sigma_1(\mathbf{A}_0)\leq \sqrt{r}+\sqrt{k}$, which enables us to get the desirable result.
\end{proof}
In the next, we show the smoothness guarantee for LoRA-FA.
\begin{lemma}\label{lemma_5}
    Let Assumption~\ref{assum_1} hold. Suppose that $\mathbf{W}$ satisfies LoRA-FA, i.e., $\mathbf{W}=\mathbf{W}_0+\frac{\eta}{r}\mathbf{B}\mathbf{A}_0$. Then, we have the following relationship:
    \begin{equation}
    \begin{split}
        &\|\nabla f(\mathbf{B}^1,\mathbf{A}_0,\mathbf{W}_0)-\nabla f(\mathbf{B}^2,\mathbf{A}_0,\mathbf{W}_0)\|_F\\&\leq \frac{\eta L}{r}\|\mathbf{A}_0\|_F^2\|\mathbf{B}^1-\mathbf{B}^2\|_F,
    \end{split}
    \end{equation}
for any given $\mathbf{B}^1$ and $\mathbf{B}^2$.
\end{lemma}
\begin{proof}
    Denote by $\nabla_{\mathbf{W}}f$ the gradient of $f$ on $\mathbf{W}$. Then we can obtain the gradient on $\mathbf{B}$ being $\frac{\eta}{r}\nabla_{\mathbf{W}}f\mathbf{A}_0^\top$ based on the chain rule. Given this and fact that $f(\mathbf{W}^i):=f(\mathbf{W}_0+\frac{\eta}{r}\mathbf{B}^i\mathbf{A}_0)$, $i=1,2$, we have the following relationship
    \begin{equation}
        \begin{split}
            &\|\nabla f(\mathbf{B}^1,\mathbf{A}_0,\mathbf{W}_0)-\nabla f(\mathbf{B}^2,\mathbf{A}_0,\mathbf{W}_0)\|_F\\&\leq\frac{\eta}{r}\|\nabla_{\mathbf{W}^1}f\mathbf{A}_0^\top-\nabla_{\mathbf{W}^2}f\mathbf{A}_0^\top\|_F\\&\leq \frac{\eta L}{r}\|\mathbf{W}^1-\mathbf{W}^2\|_F\|\mathbf{A}_0\|_F\\&\leq \frac{\eta^2 L}{r^2}\|\mathbf{B}^1-\mathbf{B}^2\|_F\|\mathbf{A}_0\|_F^2,
        \end{split}
    \end{equation}
where the second inequality follows from Assumption~\ref{assum_1}. This completes the proof.
\end{proof}

Immediately, based on Lemma~\ref{lemma_4}, it is obtained that the loss $f$ is smooth with a new constant $\eta^2L(1+\frac{k}{r}+2\frac{\sqrt{k}}{r^{1.5}})$ when using LoRA-FA. Denote by $\tilde{L} = \eta^2L(1+\frac{k}{r}+2\frac{\sqrt{k}}{r^{1.5}})$ the new smoothness constant for DLoRA-FA. We can immediately obtain the following result.

\begin{theorem}\label{dlora-fa-theo}
    Let Assumptions~\ref{assum_1},~\ref{assum_4} and~\ref{assum_5} hold. If the step size $\alpha\leq \frac{1-\sqrt{\rho}}{4\sqrt{2}\tilde{L}},$ in Algorithm~\ref{alg:dlora_fa}, then for all $T\geq 1$, the following relationship holds true:
    \begin{equation}
    \begin{split}
        \frac{1}{T}\sum_{t=1}^T\mathbb{E}[\|\nabla f(\bar{\bm{v}}_t,\bm{\theta}_0)\|^2]&\leq \frac{2D}{\alpha T}+\frac{\tilde{L}\alpha\zeta^2}{N}+\frac{\tilde{L}^2\sum_{i=1}^N\|\bm{v}_0^i\|^2}{TN(1-\rho)}+\frac{16\alpha^2\kappa^2\tilde{L}^2}{(1-\sqrt{\rho})^2}+\frac{4\alpha^2\zeta^2\tilde{L}^2}{1-\rho},
    \end{split}
    \end{equation}
where $D=f(\bar{\bm{v}}_0,\bm{\theta}_0)-f^*)$, $\tilde{L} =\eta^2L(1+\frac{k}{r}+2\frac{\sqrt{k}}{r^{1.5}})$.
\end{theorem}
The proof can follow similarly from the proof of Theorem~\ref{dlora-theo} with the new smoothness constant. Subsequently, the following corollary is also attained.
\begin{corollary}\label{corollary_2}
    Let Assumptions~\ref{assum_1}, ~\ref{assum_4} and~\ref{assum_5} hold. If step size $\alpha\lesssim\sqrt{\frac{N}{T}}$ in Algorithm~\ref{alg:dlora_fa}, then for all $T\geq\frac{32N\tilde{L}^2}{(1-\sqrt{\rho})^2}$, we have
    \begin{equation}
    \begin{split}
        &\frac{1}{T}\sum_{t=1}^T\mathbb{E}[\|\nabla f(\bar{\bm{v}}_t,\bm{\theta}_0)\|^2]\lesssim \frac{1}{\sqrt{NT}}+\frac{k}{r\sqrt{NT}}+\frac{Nk^2}{(1-\rho)Tr^2}+\frac{Nk^2}{(1-\sqrt{\rho})^2Tr^2}.
    \end{split}
    \end{equation}
\end{corollary}
The proof for the above corollary follows analogously from that in Corollary~\ref{corollary_1}. One immediate observation is that the size ratio from the frozen $\mathbf{A}$ matrix also dictates the convergence error. This can be defined as a compression ratio that would affect the convergence, which resembles the similar finding from~\cite{hao2024flora}. Comparing Corollary~\ref{corollary_1} and Corollary~\ref{corollary_2} reveals that freezing $\mathbf{A}$ matrix negatively affects the performance since the error bound of DLoRA-FA is larger than that of DLoRA due to $1\ll\frac{k}{r}$, though practically DLoRA is more computationally efficient with less communication overhead, highlighting the tradeoff between accuracy and efficiency.

We have known from Definition~\ref{definition_1} that  the model consensus interference is defined as $\mathcal{E}_t=\mathbb{E}[\|\frac{1}{N^2}\sum_{i=1}^N\mathbf{B}^i_t\sum_{i=1}^N\mathbf{A}^i_t-\frac{1}{N}\sum_{i=1}^N\mathbf{B}^i_t\mathbf{A}^i_t\|_F]$, if we consider fully-connected network. The following theorem shows the condition for eliminating the interference.
\begin{theorem}\label{model_interference_theo}
    Consider a networked system involving $N$ agents. Let $\mathbf{A}^i,\mathbf{B}^i$ and $\mathbf{A}^j, \mathbf{B}^j$ be the LoRA parameters of any two agents $i$ and $j$, respectively. Hence, the error $\mathcal{E}_t=0$ for any time step $t$ when $\mathbf{A}^i=\mathbf{A}^j$ or $\mathbf{B}^i=\mathbf{B}^j$.
\end{theorem}
\begin{proof}
    \begin{equation}
        \begin{split}
            &\mathcal{E}_t = \mathbb{E}[\|\frac{1}{N^2}\sum_{i=1}^N\mathbf{B}^i_t\sum_{i=1}^N\mathbf{A}^i_t-\frac{1}{N}\sum_{i=1}^N\mathbf{B}^i_t\mathbf{A}^i_t\|_F]\\&=\mathbb{E}[\|\frac{1}{N^2}(\sum_{i=1}^N\mathbf{B}^i_t\sum_{i=1}^N\mathbf{A}^i_t-N\sum_{i=1}^N\mathbf{B}^i_t\mathbf{A}^i_t)\|_F]\\&=\mathbb{E}[\|\frac{1}{N^2}[\sum_{j=1}^N(\mathbf{B}^j_t\sum_{i=1}^N\mathbf{A}^i_t-\sum_{i=1}^N\mathbf{B}^i_t\mathbf{A}^i_t)]\|_F]\\&=\mathbb{E}[\|\frac{1}{N^2}[\sum_{j=1}^N(\sum_{i=1}^N\mathbf{B}^j_t\mathbf{A}^i_t-\sum_{i=1}^N\mathbf{B}^i_t\mathbf{A}^i_t)]\|_F]\\&=\mathbb{E}[\|\frac{1}{N^2}[\sum_{j=1}^N\sum_{i=1}^N(\mathbf{B}^j_t-\mathbf{B}^i_t)\mathbf{A}^i_t]\|_F].
        \end{split}
    \end{equation}
Similarly, we also have $\mathcal{E}_t = \mathbb{E}[\|\frac{1}{N^2}[\sum_{i=1}^N\sum_{j=1}^N\mathbf{B}^i_t(\mathbf{A}^j_t-\mathbf{A}^i_t)]\|_F]$.
\end{proof}
Thus, $\mathcal{E}_t=0$ when $\mathbf{A}^j_t=\mathbf{A}^i_t$ or $\mathbf{B}^j_t=\mathbf{B}^i_t$ for all $i,j\in\mathcal{V}$ and $t$.
\begin{algorithm}
  \caption{\textsc{DLoRA-MSGD} (\textcolor{blue}{DeCAF-MSGD})}
  \label{alg:dlora-msgd}
  \SetKwInOut{Input}{Input}
  \SetKwInOut{Output}{Output}
  \Input{mixing matrix $\mathbf{\Pi}=[\pi_{ij}]_{i,j\in N}, \pi_{ij}\in[0,1]$, the \# of iterations $T$, initialization of low rank adapters $\mathbf{A}_1^i, \mathbf{B}_1^i, \mathbf{V}_1^{i,\mathbf{A}}, \mathbf{V}_1^{i,\mathbf{B}}, \forall i\in\mathcal{V}$, step size $\alpha$, frozen parameters $\mathbf{W}_0$, $\mathcal{S}_i, i\in\mathcal{V}$, communication frequency $\tau$, $\beta$, TSVD operator $\mathcal{T}(\cdot)$}
  \Output{$\{\mathbf{A}^i_T, \mathbf{B}^i_T\}_{i=1}^N$}  
  \BlankLine
  \For{ $t$ in $1:T$ }
  { 
    \For{each agent $i\in\mathcal{V}$}
    { 
    Calculate the stochastic gradients $\bm{g}^{i,\mathbf{\bullet}}_t=\frac{1}{|\mathcal{S}_i|}\sum_{s\in\mathcal{S}_i}\nabla_\mathbf{\bullet} f^i_s(\bm{w}^i_t,\bm{\theta}_0), \bullet\in\{\mathbf{A},\mathbf{B}\}$\;
    \eIf{$t$ mod $\tau$=0}
    {    Broadcast the low-rank matrices $\mathbf{A}^i_t, \mathbf{B}^i_t$ to and receive $\mathbf{A}^j_t, \mathbf{B}^j_t$ from all nodes in $Nb(i)$\;
        \# Individual consensus in DLoRA\; $\mathbf{A}^i_{t+1/2}=\sum_{j\in Nb(i)}\pi_{ij}\mathbf{A}^j_{t}$\;
    $\mathbf{B}^i_{t+1/2}=\sum_{j\in Nb(i)}\pi_{ij}\mathbf{B}^j_{t}$\;
    \textcolor{blue}{
    \# Product consensus in DeCAF\;
    $\tilde{\mathbf{B}}^i_t\tilde{\mathbf{A}}^i_t=\sum_{j\in Nb(i)}\pi_{ij}\mathbf{B}^j_t\mathbf{A}^j_t$\;
    $\mathbf{A}^i_{t+1/2}, \mathbf{B}^i_{t+1/2} = \mathcal{T}(\tilde{\mathbf{B}}^i_t\tilde{\mathbf{A}}^i_t)$\;
    }
    }{$\mathbf{A}^i_{t+1/2}=\mathbf{A}^i_{t}$\;
    $\mathbf{B}^i_{t+1/2}=\mathbf{B}^i_{t}$\;
    }
    $\mathbf{V}^{i,\mathbf{A}}_{t+1} = \beta\mathbf{V}^{i,\mathbf{A}}_{t}-\alpha \bm{g}^{i,\mathbf{A}}_t$\;
    $\mathbf{V}^{i,\mathbf{B}}_{t+1} = \beta\mathbf{V}^{i,\mathbf{B}}_{t}-\alpha \bm{g}^{i,\mathbf{B}}_t$\;
    $\mathbf{A}^i_{t+1}=\mathbf{A}^i_{t+1/2}+\mathbf{V}^{i,\mathbf{A}}_{t+1}$\;
    $\mathbf{B}^i_{t+1}=\mathbf{B}^i_{t+1/2}-\mathbf{V}^{i,\mathbf{B}}_{t+1}$\;    }
  }
\end{algorithm}

\begin{algorithm}
  \caption{\textsc{DLoRA-Adam} (\textcolor{blue}{DeCAF-Adam})}
  \label{alg:dlora-adam}
  \SetKwInOut{Input}{Input}
  \SetKwInOut{Output}{Output}
  \Input{mixing matrix $\mathbf{\Pi}=[\pi_{ij}]_{i,j\in N}, \pi_{ij}\in[0,1]$, the \# of iterations $T$, initialization of low rank adapters $\mathbf{A}_1^i, \mathbf{B}_1^i, \mathbf{M}_0^{i,\mathbf{A}}= \mathbf{M}_0^{i,\mathbf{B}}=
  \mathbf{V}_0^{i,\mathbf{A}}=\mathbf{V}_0^{i,\mathbf{B}}=0, \hat{\mathbf{U}}^{i,\mathbf{A}}_1 = \mathbf{V}_1^{i,\mathbf{A}}, \hat{\mathbf{U}}^{i,\mathbf{B}}_1 = \mathbf{V}_1^{i,\mathbf{B}}, \forall i\in\mathcal{V}$, step size $\alpha$, frozen parameters $\mathbf{W}_0$, $\mathcal{S}_i, i\in\mathcal{V}$, communication frequency $\tau$, $\beta_1\in[0,1), \epsilon$, TSVD operator $\mathcal{T}(\cdot)$}
  \Output{$\{\mathbf{A}^i_T, \mathbf{B}^i_T\}_{i=1}^N$}  
  \BlankLine
  \For{ $t$ in $1:T$ }
  { 
    \For{each agent $i\in\mathcal{V}$}
    { 
    Calculate the stochastic gradients $\bm{g}^{i,\mathbf{\bullet}}_t=\frac{1}{|\mathcal{S}_i|}\sum_{s\in\mathcal{S}_i}\nabla_\mathbf{\bullet} f^i_s(\bm{w}^i_t,\mathbf{W}_0), \bullet\in\{\mathbf{A},\mathbf{B}\}$\;
    $\mathbf{M}_t^{i,\mathbf{A}}=\beta_1\mathbf{M}_{t-1}^{i,\mathbf{A}}+(1-\beta_1)\bm{g}^{i,\mathbf{A}}_t$\;
    $\mathbf{M}_t^{i,\mathbf{B}}=\beta_1\mathbf{M}_{t-1}^{i,\mathbf{B}}+(1-\beta_1)\bm{g}^{i,\mathbf{B}}_t$\;
    $\mathbf{V}_t^{i,\mathbf{A}}=h_t(\bm{g}^{i,\mathbf{A}}_1,...,\bm{g}^{i,\mathbf{A}}_t)$\;
    $\mathbf{V}_t^{i,\mathbf{B}}=h_t(\bm{g}^{i,\mathbf{B}}_1,...,\bm{g}^{i,\mathbf{B}}_t)$\;
    \eIf{$t$ mod $\tau$=0}
    {     Broadcast the low-rank matrices $\mathbf{A}^i_t, \mathbf{B}^i_t$ to and receive $\mathbf{A}^j_t, \mathbf{B}^j_t$ from all nodes in $Nb(i)$\;
            \# Individual consensus in DLoRA\;
    $\mathbf{A}^i_{t+1/2}=\sum_{j\in Nb(i)}\pi_{ij}\mathbf{A}^j_{t}$\;
    $\mathbf{B}^i_{t+1/2}=\sum_{j\in Nb(i)}\pi_{ij}\mathbf{B}^j_{t}$\;
    $\hat{\mathbf{U}}^{i,\mathbf{A}}_{t+1/2} = \sum_{j\in Nb(i)}\pi_{ij}\hat{\mathbf{U}}^{j,\mathbf{A}}_{t}$\;
    $\hat{\mathbf{U}}^{i,\mathbf{B}}_{t+1/2}= \sum_{j\in Nb(i)}\pi_{ij}\hat{\mathbf{U}}^{j,\mathbf{B}}_{t}$\;
    \textcolor{blue}{
    \# Product consensus in DeCAF\;
    $\tilde{\mathbf{B}}^i_t\tilde{\mathbf{A}}^i_t=\sum_{j\in Nb(i)}\pi_{ij}\mathbf{B}^j_t\mathbf{A}^j_t$\;
    $\mathbf{A}^i_{t+1/2}, \mathbf{B}^i_{t+1/2} = \mathcal{T}(\tilde{\mathbf{B}}^i_t\tilde{\mathbf{A}}^i_t)$\;
    $\tilde{\hat{\mathbf{U}}}^{i,\mathbf{B}}_t\tilde{\hat{\mathbf{U}}}^{i,\mathbf{A}}_t=\sum_{j\in Nb(i)}\pi_{ij}\hat{\mathbf{U}}^{j,\mathbf{B}}_t\hat{\mathbf{U}}^{j,\mathbf{A}}_t$\;
    $\hat{\mathbf{U}}^{i,\mathbf{A}}_{t+1/2},\hat{\mathbf{U}}^{i,\mathbf{B}}_{t+1/2}=\mathcal{T}(\tilde{\hat{\mathbf{U}}}^{i,\mathbf{B}}_t\tilde{\hat{\mathbf{U}}}^{i,\mathbf{A}}_t)$
    }
    }{$\mathbf{A}^i_{t+1/2}=\mathbf{A}^i_{t}$\;
    $\mathbf{B}^i_{t+1/2}=\mathbf{A}^i_{t}$\;
    $\hat{\mathbf{U}}^{i,\mathbf{A}}_{t+1/2} = \hat{\mathbf{U}}^{i,\mathbf{A}}_{t}$\;
    $\hat{\mathbf{U}}^{i,\mathbf{B}}_{t+1/2} = \hat{\mathbf{U}}^{i,\mathbf{B}}_{t}$\;
    }
    $\mathbf{U}^{i,\mathbf{A}}_{t}=\text{max}(\hat{\mathbf{U}}^{i,\mathbf{A}}_{t},\epsilon)$\;
    $\mathbf{U}^{i,\mathbf{B}}_{t}=\text{max}(\hat{\mathbf{U}}^{i,\mathbf{B}}_{t},\epsilon)$\;
    $\mathbf{A}^{i}_{t+1} = \mathbf{A}^{i}_{t+1/2}-\alpha \mathbf{M}_t^{i,\mathbf{A}}\oslash (\mathbf{U}^{i,\mathbf{A}}_{t})^{1/2}$\;
    $\mathbf{B}^{i}_{t+1} = \mathbf{B}^{i}_{t+1/2}-\alpha \mathbf{M}_t^{i,\mathbf{B}}\oslash (\mathbf{U}^{i,\mathbf{B}}_{t})^{1/2}$\;
    $\hat{\mathbf{U}}^{i,\mathbf{A}}_{t+1}=\hat{\mathbf{U}}^{i,\mathbf{A}}_{t+1/2}-\mathbf{V}_{t-1}^{i,\mathbf{A}}+\mathbf{V}_t^{i,\mathbf{A}}$\;
    $\hat{\mathbf{U}}^{i,\mathbf{B}}_{t+1}=\hat{\mathbf{U}}^{i,\mathbf{B}}_{t+1/2}-\mathbf{V}_{t-1}^{i,\mathbf{B}}+\mathbf{V}_t^{i,\mathbf{B}}$\;    }
  }
\end{algorithm}

\subsubsection{Proof for Proposition~\ref{prop_1}}
\begin{proof}
    We denote by $\mathcal{R}(\mathbf{Z})$ the rank of a matrix $\mathbf{Z}\in\mathbb{R}^{m\times n}$, where $m, n$ are arbitrary dimensions. Recall $\mathbf{W}^i_t=\mathbf{W}_0+\Delta\mathbf{W}^i_t=\mathbf{W}_0+\frac{\eta}{r}\mathbf{B}^i_t\mathbf{A}^i_t$ such that the rank of $\Delta\mathbf{W}^i_t$, $\mathcal{R}(\Delta\mathbf{W}^i_t)\leq\text{min}(\mathcal{R}(\mathbf{A}^i_t), \mathcal{R}(\mathbf{B}^i_t))$. As $\mathcal{R}(\mathbf{A}^i_t)=\mathcal{R}(\mathbf{B}^i_t)=r$, then we have $\mathcal{R}(\Delta\mathbf{W}^i_t)\leq r$. Since TSVD operator satisfies $\mathcal{T}(\tilde{\mathbf{B}}^i_t\tilde{\mathbf{A}}^i_t):=\mathbf{U}_r\Sigma_r\mathbf{V}^\top_r$, it is immediately obtained that $\mathcal{R}(\mathbf{B}^i_{t+1/2}\mathbf{A}^i_{t+1/2})=r$, where $\mathbf{B}^i_{t+1/2}=\mathbf{U}_r\Sigma^{1/2}_r$ and $\mathbf{A}^i_{t+1/2}=(\mathbf{V}_r\Sigma^{1/2}_r)^\top$. This is due to LoRA with the intrinsic rank $r$. We now look at the product consensus relationship $\tilde{\mathbf{B}}^i_t\tilde{\mathbf{A}}^i_t=\sum_{j\in Nb(i)}\pi_{ij}\mathbf{B}^j_t\mathbf{A}^j_t$ and have known that $\mathcal{R}(\mathbf{B}^i_t\mathbf{A}^i_t)\leq r$. Recall the subadditivity property in the matrix rank. Then, we have $\mathcal{R}(\mathbf{Z}_1+...+\mathbf{Z}_N)\leq \sum_{l=1}^N\mathcal{R}(\mathbf{Z}_l)$ for arbitrary matrix $\mathbf{Z}_l$, which results in $\mathcal{R}(\tilde{\mathbf{B}}^i_t\tilde{\mathbf{A}}^i_t)\leq |Nb(i)|r$. Please note that the weights $\pi_{ij}$ will not affect the rank of a matrix. 
    Additionally, $\mathcal{R}(\tilde{\mathbf{B}}^i_t\tilde{\mathbf{A}}^i_t)\geq r$. If $\mathcal{R}(\tilde{\mathbf{B}}^i_t\tilde{\mathbf{A}}^i_t)<r$, this contradicts the fact that in TSVD, the rank of original matrix should be at least equal to or larger than the rank of any decomposed matrix. Since TSVD is essentially a low-rank approximation, $\mathbf{B}^i_{t+1/2}\mathbf{A}^i_{t+1/2}$ is the approximation of $\tilde{\mathbf{B}}^i_t\tilde{\mathbf{A}}^i_t$ using the top $r$ components, which leads to the following approximation error:
    \begin{equation}
        \|\tilde{\mathbf{B}}^i_t\tilde{\mathbf{A}}^i_t-\mathbf{B}^i_{t+1/2}\mathbf{A}^i_{t+1/2}\|_F=\sqrt{\sum_{\tau=r+1}^{|Nb(i)|r}\sigma^2_\tau(\tilde{\mathbf{B}}^i_t\tilde{\mathbf{A}}^i_t)}\leq \sqrt{(|Nb(i)|-1)r}\sigma_1(\tilde{\mathbf{B}}^i_t\tilde{\mathbf{A}}^i_t),
    \end{equation}
    where $\sigma_\tau(\tilde{\mathbf{B}}^i_t\tilde{\mathbf{A}}^i_t)$ is the $\tau$-th largest singular value of $\tilde{\mathbf{B}}^i_t\tilde{\mathbf{A}}^i_t$.
    Pertaining to $\tilde{\mathbf{B}}^i_t\tilde{\mathbf{A}}^i_t=\sum_{j\in Nb(i)}\pi_{ij}\mathbf{B}^j_t\mathbf{A}^j_t$, with the relationships, $\sigma_{1}(\mathbf{Z}_1\mathbf{Z}_2)\leq \sigma_1(\mathbf{Z}_1)\sigma_1(\mathbf{Z}_2)$, $\sigma_1(a\mathbf{Z}_1)=|a|\sigma_1(\mathbf{Z}_1)$, where $a$ is a constant, and $\sigma_1(\sum_{l=1}^N\mathbf{Z}_l)\leq \sum_{l=1}^N\sigma_1(\mathbf{Z}_l)$, and Assumption~\ref{assum_2}, we have the following relationship:
    \begin{equation}
        \begin{split}
            \sqrt{(|Nb(i)|-1)r}\sigma_1(\tilde{\mathbf{B}}^i_t\tilde{\mathbf{A}}^i_t)&\leq \sqrt{(|Nb(i)|-1)r}(\sigma_1(\pi_{i1}\mathbf{B}^1_t\mathbf{A}^1_t)+...+\sigma_1(\pi_{i|Nb(i)|}\mathbf{B}^{|Nb(i)|}_t\mathbf{A}^{|Nb(i)|}_t))\\&\leq\sqrt{(|Nb(i)|-1)r}(\pi_{i1}\sigma_1(\mathbf{B}^1_t)\sigma_1(\mathbf{A}^1_t)+...+\pi_{i|Nb(i)|}\sigma_1(\mathbf{B}^{|Nb(i)|}_t)\sigma_1(\mathbf{A}^{|Nb(i)|}_t))\\&\leq \sqrt{(|Nb(i)|-1)r}c^2,
        \end{split}
    \end{equation}
    which yields the following relationship:
    \begin{equation}
        \|\tilde{\mathbf{B}}^i_t\tilde{\mathbf{A}}^i_t-\mathbf{B}^i_{t+1/2}\mathbf{A}^i_{t+1/2}\|_F\leq \sqrt{(|Nb(i)|-1)r}c^2,
    \end{equation}
    which completes the proof.
\end{proof}
\subsubsection{Proof for Theorem~\ref{theorem_2}}
Before presenting the proof for Theorem~\ref{theorem_2}, we briefly discuss the difference of consensus in both DLoRA and DeCAF, $\mathcal{E}_T=\mathbb{E}[\|\sum_{j\in Nb(i)}\pi_{ij}\mathbf{B}^j_T\mathbf{A}^j_T - \sum_{j\in Nb(i)}\pi_{ij}\mathbf{B}^j_T\sum_{j\in Nb(i)}\pi_{ij}\mathbf{A}^j_T\|_F]$. In DeCAF, the consensus on the product and the individual updates require TSVD, which introduces approximation error. This also poses a difficulty on direct analysis of $\mathcal{E}_T$ due to the missing consensus operator in Lines 15-16. Therefore, to theoretically shed light on the difference between DLoRA and DeCAF, we assume that there is no TSVD in DeCAF as it plays a central role in practical implementation to a greater extent. As $\mathbf{B}^i_{t}\mathbf{A}^i_{t}$ represents the change $\Delta\mathbf{W}^i_t$ to the frozen parameter $\mathbf{W}_0$, this motivates us to directly resort to $\mathbf{B}^i_{t}\mathbf{A}^i_{t}$ in the local update. However, in DLoRA and DeCAF, we have used individual update as this can significantly reduce the empirical computational and communication overhead. Therefore, the different version of DeCAF is 
\begin{equation}\label{eq_41}
    \Delta\mathbf{W}^i_{t+1/2} = \sum_{j\in Nb(i)}\pi_{ij}\Delta\mathbf{W}_t^j,\;\Delta\mathbf{W}^i_{t+1} = \Delta\mathbf{W}^i_{t+1/2}-\alpha\bm{g}^{i,\Delta\mathbf{W}}_t.
\end{equation}
Although in terms of algorithmic perspective, there exist some difference between individual updates for $\mathbf{A}^i$ and $\mathbf{B}^i$ and the update for $\Delta\mathbf{W}^i$, examining $\mathcal{E}_T$ will roughly inform us how different the DLoRA and DeCAF is and what to impact the distinction. Quantifying this error bound also assists in determining the convergence rate for DeCAF. 
\begin{proof}
    Recalling $\mathcal{E}_T=\mathbb{E}[\|\sum_{j\in Nb(i)}\pi_{ij}\mathbf{B}^j_T\mathbf{A}^j_T - \sum_{j\in Nb(i)}\pi_{ij}\mathbf{B}^j_T\sum_{j\in Nb(i)}\pi_{ij}\mathbf{A}^j_T\|_F]$,
    we have the following relationship
    \begin{equation}\label{eq_42}
    \begin{split}
        \mathcal{E}_T\leq\mathbb{E}[\|\sum_{j\in Nb(i)}\pi_{ij}\mathbf{B}^j_T\mathbf{A}^j_T\|_F] +\mathbb{E}[ \|\sum_{j\in Nb(i)}\pi_{ij}\mathbf{B}^j_T\sum_{j\in Nb(i)}\pi_{ij}\mathbf{A}^j_T\|_F]
    \end{split}
    \end{equation}
    For $\sum_{j\in Nb(i)}\pi_{ij}\mathbf{B}^j_T\mathbf{A}^j_T$, based on Eq.~\ref{eq_41}, and Lemma 4.2 in~\cite{berahas2018balancing}, and $\mathbf{B}^i_0\mathbf{A}^i_0=0$, the following relationship is obtained:
    \begin{equation}
        \sum_{j\in Nb(i)}\pi_{ij}\mathbf{B}^j_T\mathbf{A}^j_T = -\alpha\sum_{t=0}^{T-1}\rho^{\frac{T-t}{2}}\bm{g}^{i,\Delta\mathbf{W}}_{t}
    \end{equation}
    As $\bm{g}^{i,\Delta\mathbf{W}}_{t}=\frac{\eta}{r}(\mathbf{B}^i_t)^\top\bm{g}^{i,\mathbf{W}}_t + \frac{\eta}{r}\bm{g}^{i,\mathbf{W}}_t(\mathbf{A}^i_t)^\top$, substituting it into the above equality yields the following:
    \begin{equation}\label{eq_44}
        \sum_{j\in Nb(i)}\pi_{ij}\mathbf{B}^j_T\mathbf{A}^j_T = -\alpha\sum_{t=0}^{T-1}\rho^{\frac{T-t}{2}}(\frac{\eta}{r}(\mathbf{B}^i_t)^\top\bm{g}^{i,\mathbf{W}}_t + \frac{\eta}{r}\bm{g}^{i,\mathbf{W}}_t(\mathbf{A}^i_t)^\top).
    \end{equation}
    Similarly, we can obtain the following relationships for $\sum_{j\in Nb(i)}\pi_{ij}\mathbf{A}^j_T$ and $\sum_{j\in Nb(i)}\pi_{ij}\mathbf{B}^j_T$:
    \begin{equation}\label{eq_45}
        \sum_{j\in Nb(i)}\pi_{ij}\mathbf{A}^j_T = \rho^{\frac{T}{2}}\mathbf{A}^i_0-\alpha\sum_{t=0}^T\rho^{\frac{T-t}{2}}\bm{g}^{i,\mathbf{A}}_t
    \end{equation}
    \begin{equation}\label{eq_46}
        \sum_{j\in Nb(i)}\pi_{ij}\mathbf{B}^j_T = -\alpha\sum_{t=0}^T\rho^{\frac{T-t}{2}}\bm{g}^{i,\mathbf{B}}_t
    \end{equation}
This is due to $\mathbf{B}^i_0=0$. Substituting Eqs.~\ref{eq_44}, \ref{eq_45}, and~\ref{eq_46} into Eq.~\ref{eq_42} produces the following:
\begin{equation}
    \begin{split}
        \mathcal{E}_T &= \mathbb{E}[\|-\alpha\sum_{t=0}^{T-1}\rho^{\frac{T-t}{2}}(\frac{\eta}{r}(\mathbf{B}^i_t)^\top\bm{g}^{i,\mathbf{W}}_t + \frac{\eta}{r}\bm{g}^{i,\mathbf{W}}_t(\mathbf{A}^i_t)^\top)\|_F ]\\&+ \mathbb{E}[\|(\rho^{\frac{T}{2}}\mathbf{A}^i_0-\alpha\sum_{t=0}^T\rho^{\frac{T-t}{2}}\bm{g}^{i,\mathbf{A}}_t)(-\alpha\sum_{t=0}^T\rho^{\frac{T-t}{2}}\bm{g}^{i,\mathbf{B}}_t)\|_F].
    \end{split}
\end{equation}
Due to the fact that for any two matrices $\mathbf{M}$ and $\mathbf{N}$, $\|\mathbf{M}\mathbf{N}\|_F\leq\|\mathbf{M}\|_2\|\mathbf{N}\|_F$ (or $\|\mathbf{M}\mathbf{N}\|_F\leq\|\mathbf{M}\|_F\|\mathbf{N}\|_2$), $\|\mathbf{M}\|_2\leq\|\mathbf{M}\|_F$, $\|\mathbf{M}+\mathbf{N}\|_F\leq \|\mathbf{M}\|_F + \|\mathbf{N}\|_F$, and $\bm{g}^{i,\mathbf{A}}_t=\frac{\eta}{r}(\mathbf{B}^i_t)^\top\bm{g}^{i,\mathbf{W}}_t$ and $\bm{g}^{i,\mathbf{B}}_t=\frac{\eta}{r}\bm{g}^{i,\mathbf{W}}_t(\mathbf{A}^i_t)^\top$ we have
\begin{equation}
\begin{split}
    \mathcal{E}_T&\leq \alpha\sum_{t=0}^{T-1}\rho^{\frac{T-t}{2}}\mathbb{E}[\|\frac{\eta}{r}(\mathbf{B}^i_t)^\top\bm{g}^{i,\mathbf{W}}_t + \frac{\eta}{r}\bm{g}^{i,\mathbf{W}}_t(\mathbf{A}^i_t)^\top\|_F]\\&+\mathbb{E}[\|\rho^{\frac{T}{2}}\mathbf{A}^i_0-\alpha\sum_{t=0}^T\rho^{\frac{T-t}{2}}\bm{g}^{i,\mathbf{A}}_t\|_2\|\alpha\sum_{t=0}^T\rho^{\frac{T-t}{2}}\bm{g}^{i,\mathbf{B}}_t\|_F]\\&\leq \alpha\sum_{t=0}^{T-1}\rho^{\frac{T-t}{2}}(\mathbb{E}[\frac{\eta}{r}\|(\mathbf{B}^i_t)^\top\bm{g}^{i,\mathbf{W}}_t\|_F] + \mathbb{E}[\frac{\eta}{r}\|\bm{g}^{i,\mathbf{W}}_t(\mathbf{A}^i_t)^\top\|_F])\\&+(\|\rho^{\frac{T}{2}}\mathbf{A}^i_0\|_2+\mathbb{E}[\|\alpha\sum_{t=0}^T\rho^{\frac{T-t}{2}}\bm{g}^{i,\mathbf{A}}_t\|_2])\mathbb{E}[\|\alpha\sum_{t=0}^T\rho^{\frac{T-t}{2}}\bm{g}^{i,\mathbf{B}}_t\|_F]\\&\leq \frac{2\alpha G\eta c}{\sqrt{r}(1-\sqrt{\rho})} + (\rho^{\frac{T}{2}}c+\frac{\alpha\eta cG}{r(1-\sqrt{\rho})})\frac{\alpha \eta cG}{r(1-\sqrt{\rho})},
\end{split}
\end{equation}
which completes the proof.
\end{proof}

\subsubsection{Theoretical result for the centralized setting}\label{centralized_analysis}
In this subsection, we present the convergence rate for a special case when $N=1$, which is of independent interest, as we have not been aware of any existing convergence rates with addressing the smoothness guarantee issue. To characterize the analysis, we modify a couple assumptions that have been imposed for DLoRA. We still have the smoothness assumption with modulus $L$ for $f$, which has now been only one agent. Immediately, leveraging Lemma~\ref{lemma_1} yields that $f$ is also $\hat{L}$-smooth. Subsequently, we have the bounded stochastic gradient variance $\mathbb{E}[\|\bm{g}-\nabla f(\bm{v},\bm{\theta}_0)\|^2]\leq \zeta^2$. Therefore, the following theorem summarize the result for the centralized setting.
\begin{theorem}\label{theo_7}
    Let Assumptions~\ref{assum_1} to~\ref{assum_4} hold. Consider a sequence $(\bm{v}_t)_{t\in\mathbb{N}}$ generated by the centralized version of Algorithm~\ref{alg:dlora} without any communication, with a constant step size $\alpha=\sqrt{\frac{D}{\hat{L}\zeta^2T}}$, where $D=f(\bm{v}_1,\bm{\theta}_0)-f^*$, $\hat{L}=\frac{\eta(2LC\sqrt{r}\sigma+G)}{r}$. Then for any $T\geq \frac{4\hat{L}D}{\zeta^2}$, it follows that
    \begin{equation}
        \textnormal{min}_{t=1:T}\mathbb{E}[\|\nabla f(\bm{v}_t,\bm{\theta}_0)\|^2]\leq 2.5\sqrt{\frac{D\hat{L}\zeta^2}{T}}.
    \end{equation}
\end{theorem}
\begin{proof}
    In light of the smoothness assumption, we can obtain the following relationship
    \begin{equation}\label{eq_14}
    \begin{split}
        &f(\bm{v}_{t+1},\bm{\theta}_0)-f(\bm{v}_{t},\bm{\theta}_0)-\langle\nabla f(\bm{v}_t,\bm{\theta}_0),\bm{v}_{t+1}-\bm{v}_t\rangle\\&\leq\frac{\hat{L}}{2}\|\bm{v}_{t+1}-\bm{v}_t\|^2.
    \end{split}
    \end{equation}
    Substituting the update law into the last equation yields the following relationship:
    \begin{equation}\label{eq_15}
        f(\bm{v}_{t+1},\bm{\theta}_0)-f(\bm{v}_{t},\bm{\theta}_0)+\alpha\langle\nabla f(\bm{v}_t,\bm{\theta}_0),\bm{g}_t\rangle\leq \frac{\hat{L}}{2}\alpha^2\|\bm{g}_t\|^2.
    \end{equation}
    The right hand side of Eq.~\ref{eq_15} can be rewritten as
    \begin{equation}\label{eq_16}
        \frac{\hat{L}}{2}\alpha^2\|\bm{g}_t\|^2=\frac{\hat{L}}{2}\alpha^2\|\bm{g}_t-\nabla f(\bm{v}_t,\bm{\theta}_0)+\nabla f(\bm{v}_t,\bm{\theta}_0)\|^2,
    \end{equation}
    which leads to the following relationship
    \begin{equation}\label{eq_17}
        \frac{\hat{L}}{2}\alpha^2\|\bm{g}_t\|^2\leq \hat{L}\alpha^2(\|\bm{g}_t-\nabla f(\bm{v}_t,\bm{\theta}_0)\|^2+\|\nabla f(\bm{v}_t,\bm{\theta}_0)\|^2)
    \end{equation}
    It follows from a basic inequality $\|\bm{a}+\bm{b}\|^2\leq 2\|\bm{a}\|^2 + 2\|\bm{b}\|^2$. Substituting Eq.~\ref{eq_17} to Eq.~\ref{eq_15} and taking expectation on both sides produces:
    \begin{equation}\label{eq_18}
    \begin{split}
        &\mathbb{E}[f(\bm{v}_{t+1},\bm{\theta}_0)-f(\bm{v}_{t},\bm{\theta}_0)]+\alpha\mathbb{E}[\|\nabla f(\bm{v}_t,\bm{\theta}_0)\|^2]\leq \\&\hat{L}\alpha^2\mathbb{E}[\|\bm{g}_t-\nabla f(\bm{v}_t,\bm{\theta}_0)\|^2+\|\nabla f(\bm{v}_t,\bm{\theta}_0)\|^2].
    \end{split}
    \end{equation}
    By leveraging Assumption~\ref{assum_4}, Eq.~\ref{eq_18} becomes
    \begin{equation}\label{eq_19}
    \begin{split}
        &\mathbb{E}[f(\bm{v}_{t+1},\bm{\theta}_0)-f(\bm{v}_{t},\bm{\theta}_0)]+\alpha\mathbb{E}[\|\nabla f(\bm{v}_t,\bm{\theta}_0)\|^2]\\&\leq \hat{L}\alpha^2\zeta^2+\hat{L}\alpha^2\mathbb{E}[\|\nabla f(\bm{v}_t,\bm{\theta}_0)\|^2]
    \end{split}
    \end{equation}
    As the step size $\alpha=\sqrt{\frac{D}{\hat{L}\zeta^2T}}$ and $T\geq \frac{4\hat{L}D}{\zeta^2}$, we can obtain that $\alpha\leq \frac{1}{2\hat{L}}$. We then have:
    \begin{equation}\label{eq_20}
        \mathbb{E}[f(\bm{v}_{t+1},\bm{\theta}_0)-f(\bm{v}_{t},\bm{\theta}_0)]\leq \hat{L}\alpha^2\zeta^2 - \frac{\alpha}{2}\mathbb{E}[\|\nabla f(\bm{v}_t,\bm{\theta}_0)\|^2].
    \end{equation}
    With some simple mathematical manipulations, the following inequality is obtained:
    \begin{equation}\label{eq_21}
        \mathbb{E}[\|\nabla f(\bm{v}_t,\bm{\theta}_0)\|^2]\leq \frac{2(\mathbb{E}[f(\bm{v}_{t},\bm{\theta}_0)-f(\bm{v}_{t+1},\bm{\theta}_0)])}{\alpha}+\frac{\hat{L}\alpha\zeta^2}{2}.
    \end{equation}
    Now we sum the above equation over 1 to $T$ such that
    \begin{equation}\label{eq_22}
        \sum_{t=1}^T\mathbb{E}[\|\nabla f(\bm{v}_t,\bm{\theta}_0)\|^2]\leq\frac{2(f(\bm{v}_{1},\bm{\theta}_0)-f^*)}{\alpha}+\frac{T\hat{L}\alpha\zeta^2}{2}.
    \end{equation}
    Dividing both sides by $T$, we have:
    \begin{equation}\label{eq_23}
    \begin{split}
    \textnormal{min}_{t=1:T}\mathbb{E}[\|\nabla f(\bm{v}_t,\bm{\theta}_0)\|^2]&\leq \frac{\sum_{t=1}^T\mathbb{E}[\|\nabla f(\bm{v}_t,\bm{\theta}_0)\|^2]}{T}\\&\leq \frac{2D}{\alpha T}+\frac{\hat{L}\alpha\zeta^2}{2}.
    \end{split}
    \end{equation}
    Substituting the step size $\alpha=\sqrt{\frac{D}{\hat{L}\zeta^2T}}$ into the last inequality yields the desirable result.
\end{proof}
From the above conclusion, the initialization error $D$ and the variance of $\bm{g}$ also influence the error bound. Both a good initialization and a large batch are able to reduce the convergence error bound, making it reach an "approximate" critical point faster. 

\subsubsection{Algorithm frameworks for Momentum-based SGD and Adam}
Algorithm~\ref{alg:dlora-msgd} shows DLoRA-MSGD and DeCAF-MSGD, while Algorithm~\ref{alg:dlora-adam} presents DLoRA-Adam and DeCAF-Adam. $h_t$ in Algorithm~\ref{alg:dlora-adam} can take different forms, leading to different variants such as AMSGrad.

\subsection{Additional Numerical Results}\label{additional_nu_results}

In this subsection, we provide a detailed explanation of our mixing matrix. We also include the results from instruction tuning and present a comparative analysis of the computation and communication overhead associated with different ranks. This analysis will compare DeCAF, DLoRA-FA, and DLoRA.

\subsubsection{Mixing Matrix}
Our mixing matrix is a doubly stochastic matrix that contains values (\(\pi\)) representing the inter-agent influences in collaborative learning. For simplicity, in a fully connected topology, all elements are uniform (for example, 0.1 for 10 agents). In a ring topology, each agent equally influences its two neighbors as well as itself (for instance, 0.333 for three adjacent agents)~\cite{jiang2017collaborative,esfandiari2021cross}.

\subsubsection{Instruction Tuning Results for Large Language Models (LLMs)}

To further investigate the effectiveness of instruction tuning in collaborative learning environments, we implemented instruction tuning in a decentralized setting for the first time. While we have not yet integrated the DeCAF algorithm into this setup (ongoing work), we evaluated our proposed DLoRA-FA algorithm, which consistently ranks as either the best or second-best performer across most metrics. Table~\ref{tab:init_llm} summarizes the performance of various algorithms under the Alpaca-GPT4 IID\cite{peng2023instructiontuninggpt4} setting, involving a network of 10 decentralized agents. The base model used was the pre-trained LLaMA2-7B\cite{touvron2023llama2openfoundation}, fine-tuned using LoRA with a rank of 16, a scaling factor $\eta$ of 64, and a dropout rate of 0.1. Training was conducted with a batch size of 16, using a cosine learning rate scheduler over 2000 batches. Communication between agents occurred every 10 batches in both federated and decentralized configurations. Due to computational constraints, we did not include centralized training baselines in this evaluation. We report results using three key metrics:

MTBench\cite{zheng2023judgingllmasajudgemtbenchchatbot} (MT), which assesses the model’s ability to handle multi-turn, open-ended conversations judged by GPT-4o\cite{openai2024gpt4ocard}. Higher scores indicate better conversational coherence and contextual understanding.

AdvBench~\cite{zou2023universaltransferableadversarialattacks} (Adv), which evaluates the model’s robustness against adversarial prompts designed to elicit unsafe or misaligned responses. Higher scores reflect stronger alignment and safety.

Vicuna~\cite{zheng2023judgingllmasajudgemtbenchchatbot} (Vic), which benchmarks the model's performance specifically within Vicuna-style evaluation frameworks. Higher values indicate better alignment with Vicuna-optimized architectures.

\begin{table*}[htbp]
    \centering
    \caption{\small{Different evaluation metrics on Alpaca-GPT4 IID using LLAMA2-7B, trained with 10 agents over 2000 batches on a Fully Connected topology. Best and second-best scores are in bold and underlined.}}
    \label{tab:init_llm}
    \begin{tabular}{lccccccc}
        \toprule
        \textbf{Data} &
        \textbf{Local} & \textbf{Local-FA} & \textbf{FedAvg} 
        & \textbf{FedAvg-FA} & \textbf{DLoRA} & \textbf{DLoRA-FA} 
        \\
        \midrule
        \textbf{MT}   & $3.5$  & \underline{$3.8$} & $3.7$  & $3.7$  & \boldmath{$4.1$} & \underline{$3.8$} 
        \\
        \textbf{Adv}  & $44.8$  & \boldmath{$56.3$} & \underline{51.7}  & $43.1$  & $48.1$ & $50.8$ 
        \\
        \textbf{Vic}  & $6.2$   & \underline{$6.8$} & $6.7$ & $6.6$  & \underline{$6.8$} & \boldmath{$6.9$} 
        \\
        \bottomrule
    \end{tabular}
\end{table*}
\begin{table}[htbp]
    \centering
    \caption{Training Time (TT) in seconds and Number of Trainable Parameters (NTP) for both VLMs and LLMs using 10 agents in a fully connected topology under an IID setting, evaluated with the DeCAF algorithm. VLM experiments used the Flower dataset with the CLIP model, while LLM experiments used the WIC dataset with the LLAMA2-7B model.}
    \label{tab:com_results}
    \begin{tabular}{lcccc}
        \toprule
        \multirow{2}{*}{\textbf{Low Rank}} & \multicolumn{2}{c}{\textbf{VLMs}} & \multicolumn{2}{c}{\textbf{LLMs}} \\
        \cmidrule(lr){2-3} \cmidrule(lr){4-5}
        & \textbf{TT} & \textbf{NTP} & \textbf{TT} & \textbf{NTP}  \\
        \midrule
        \textbf{1} & 15.62 & 92,160 & 2241.75 & 524,288 \\
        \textbf{4} & 15.64 & 368,640 & 2520.46 & 2,097,152 \\
        \textbf{16} & 15.70 & 1,474,560 & 3367.96 & 8,388,608 \\
        \textbf{64} & 15.58 & 5,898,240 & 4088.95 & 33,554,432 \\
        \bottomrule
    \end{tabular}
\end{table}

\begin{table}[htbp]
    \centering
    \caption{Training Time (TT) in seconds and Number of Trainable Parameters (NTP) for both VLMs and LLMs using 10 agents in a fully connected topology under an IID setting. The VLM evaluation used the Flower dataset with the CLIP model (Low Rank: 2), and the LLM used the WIC dataset with the LLAMA2-7B model (Low Rank: 4).}
    \label{tab:com_results_fa}
    \begin{tabular}{lcccc}
        \toprule
        \multirow{2}{*}{\textbf{Method}} & \multicolumn{2}{c}{\textbf{VLMs}} & \multicolumn{2}{c}{\textbf{LLMs}} \\
        \cmidrule(lr){2-3} \cmidrule(lr){4-5}
        & \textbf{TT} & \textbf{NTP} & \textbf{TT} & \textbf{NTP}  \\
        \midrule
        \textbf{DeCAF} & 15.21 & 184,320 & 2520.46 & 2,097,152 \\
        \textbf{DLoRA-FA} & 14.99 & 92,160 & 248.63 & 1,048,576 \\
        \textbf{DLoRA} & 15.22 & 184,320 & 251.82 & 2,097,152 \\
        \bottomrule
    \end{tabular}
\end{table}

\subsubsection{Computation and Communication Overhead}
We present the computation and communication overhead results of our proposed algorithm, evaluated on both VLMs and LLMs using 10 agents in a fully connected topology under an IID setting. The VLM experiments used the Flower dataset with the CLIP model, while the LLM experiments used the WIC dataset with the LLAMA2-7B model. For each setup, we report the Training Time (TT) in seconds for a single communication and training step, as well as the Number of Trainable Parameters (NTP). All experiments were conducted on NVIDIA A100 GPUs.

\noindent
\textbf{Low Rank Configuration:}  
Table~\ref{tab:com_results} shows the impact of varying the LoRA rank on computation and communication costs in the DeCAF setup. While training time remains almost constant for VLMs due to their smaller model size, LLMs show increasing training time with higher ranks. This is attributed to the increase in the number of trainable parameters and additional computations such as TSVD, confirming that computational overhead scales with model complexity and rank.

\noindent
\textbf{Comparison of Decentralized Methods:}  
Table~\ref{tab:com_results_fa} compares three decentralized LoRA methods: DeCAF, DLoRA-FA, and DLoRA. Among these, DLoRA-FA shows the lowest communication cost, making it more efficient in bandwidth-constrained settings. However, this reduction in communication comes at the expense of accuracy, aligning with our theoretical insights about the trade-offs in communication-efficient learning.

In terms of computation, the differences are negligible for VLMs due to their small size. However, for LLMs, DeCAF incurs significantly higher training time, primarily due to the additional TSVD step. Despite this overhead, DeCAF often yields better robustness and accuracy, as demonstrated in our other experiments, making it a compelling choice for high-stakes applications. Further optimization of the TSVD step remains an open area for future work to enhance the practical usability of DeCAF in large-scale scenarios.

\end{document}